\documentclass{article}

 \usepackage[preprint]{neurips_2026}

\usepackage{amsmath}
\usepackage{amssymb}
\usepackage{amsthm}
\usepackage{booktabs}
\usepackage{pifont}
\usepackage{hyperref}
\usepackage{booktabs}
\usepackage{tabularx}
\usepackage{array}
\usepackage{graphicx}
\usepackage{graphicx}
\usepackage{subcaption}
\newcolumntype{L}[1]{>{\raggedright\arraybackslash}p{#1}}
\newcolumntype{Y}{>{\raggedright\arraybackslash}X}

\usepackage{soul}

\usepackage{subcaption} 

\newtheorem{theorem}{Theorem}
\newtheorem{proposition}{Proposition}

\newtheorem{corollary}{Corollary}

\usepackage{algorithm}
\usepackage{algorithmic}
\usepackage{multirow}
\usepackage[table]{xcolor}
\usepackage{wrapfig}

\newcommand{\Eb}{\mathbb{E}}

\newtheorem{definition}{Definition}

\usepackage{subcaption} 

\newcommand{\overm}{\frac{1}{m}}
\newcommand{\summ}{\sum_{i=1}^m}
\newcommand{\cL}{\mathcal{L}}

\newcommand{\norm}[1]{\left\lVert#1\right\rVert}

\usepackage[textsize=tiny]{todonotes}

\title{Core-Halo Decomposition: Decentralizing Large-Scale Fixed-Point Problems}
\author{%
    Haixiang Sun\thanks{Equal contribution.}\\
  Purdue University\\
  \texttt{sun1321@purdue.edu} \\
  \And
  Yang Xu\footnotemark[1]\\
  Purdue University\\
  \texttt{xu1720@purdue.edu} \\
  \And
  Jiefu Zhang\\
  Purdue University\\
  \texttt{zhan4018@purdue.edu} \\
  \And
  Xudong Wu \\
  The University of Hong Kong \\
  \texttt{wu.xudong@connect.hku.hk}\\
  \And
  Zihan Zhou\\
  Johns Hopkins University\\
  \texttt{zzhou150@jh.edu} \\
  \And
  Jun He \\
  Purdue University \\
  \texttt{he184@purdue.edu}
  \And
  Jiayu Chen \\
  The University of Hong Kong \\
  \texttt{jiayuc@hku.hk}
}

\begin{document}

\maketitle

\begin{abstract}
We study solving large-scale fixed-point equation
\(x^\star=\bar F(x^\star)\) with decomposition. Standard strict decomposition assigns each agent a disjoint block and evaluates updates using only owned coordinates. For most operators, however, a block update may depend on variables outside the block. Truncating these dependencies by strict decomposition changes the mean
operator and creates structural bias that cannot be removed by more samples, smaller stepsizes, or additional consensus. We therefore propose Core-Halo decomposition, which separates write ownership from read-only evaluation context: each agent updates its own core and reads from an overlapping halo. By aligning the Core-Halo decomposition with the block-dependence structure of $\bar F$, the original fixed-point problem can be implemented faithfully in a decentralized multi-agent system. We further characterize the fundamental obstruction faced by strict decomposition through a Bellman closure condition and a blockwise bias lower bound, showing that local-only updates can alter the original fixed-point operator. Finally, we conduct extensive experiments across a range of application settings, and demonstrate that Core-Halo achieves near-centralized performance while retaining the parallelism benefits of decentralization.
\end{abstract}

\section{Introduction}
Large-scale learning and control systems are increasingly deployed in settings where data, computation, and decisions are distributed across many agents. A common mathematical abstraction behind these systems is a global fixed-point equation. This viewpoint appears in dynamic programming and Bellman equations \cite{bellman1957dynamic,watkins1992q,tsitsiklis1994asynchronous,jaakkola1994convergence}, temporal-difference and SARSA-type reinforcement learning \cite{sutton1988learning,rummery1994line,vanseijen2009expected}, PageRank-style graph computation \cite{Page1998PageRank,gleich2015pagerank}, model predictive control and sparse networked control \cite{rawlings2017model,scattolini2009architectures,christofides2013distributed}, and implicit equilibrium models in deep learning \cite{bai2019deq,winston2020mondeq,ghaoui2021implicit,fung2022jfb,sun2024understanding}.To solve these fixed-point equations, stochastic approximation provides the standard sample-path mechanism for solving such fixed-point equations when only noisy evaluations are available \cite{robbins1951stochastic,kushner2003stochastic,borkar2008stochastic}. We refer the fixed-point problem as follows: let $F: \mathbb{R}^d \times \mathcal{X} \to \mathbb{R}^d$ be a stochastic operator, where $\mathcal{X}$ is a sample space with distribution $\mu$. The goal is to find $x^\star$ satisfying
\begin{equation}\label{eq:problem_intro}
    x^\star = \bar F(x^\star), \tag{P}
\end{equation}
where $\bar F(x) := \mathbb{E}_{\xi\sim\mu}[F(x,\xi)]$. This formulation is broadly applicable: in reinforcement learning, $\bar F$ may be a Bellman mean operator; in PageRank, it is a damped graph-linear map; and in implicit models, it is the equilibrium map whose fixed point defines the layer output.

Decentralization is a natural way to tackle large-scale problems when data, computation, or decisions are distributed across many agents. Early work on distributed fixed-point and parallel computation formalized asynchronous and networked updates \cite{bertsekas1983distributed,bertsekas1989parallel}, while distributed stochastic approximation studied noisy local recursions coupled with communication \cite{kushner1987asymptotic,stankovic2011decentralized,stankovic2016distributed}.

Most of existing works assume an explicitly decomposable structure. A prototypical example is a finite-sum mean operator $\bar F(x)=\frac{1}{m}\sum_{i=1}^m \bar F_i(x)$, where agent $i$ owns a local stochastic oracle for $\bar F_i$. This model is standard in distributed optimization \cite{nedic2009distributed,boyd2011distributed} and has also been adopted in decentralized reinforcement learning and stochastic approximation \cite{kar2013qd,zhang2018fully,doan2019finite}. Recent finite-time analyses further quantify the interaction between network mixing and stochastic error in this setting \cite{sun2020finite,wang2020decentralized,zeng2023finite}. In this benchmark setting, decentralized methods can reduce the dominant stochastic term through parallel sampling while preserving the contraction-driven bias decay of a serial method. In this benchmark setting, decentralized methods can reduce the dominant stochastic term by parallelism while preserving the contraction-driven bias decay of a serial method. Specifically, we consider multi-agent Decentralized Stochastic Approximation (DSA) 
$$
\begin{aligned}
    & x_i^{k+1} = \sum_{j=1}^m w_{ij}(x_j^k + \alpha(F_j(x_j^k, \xi_j^k)-x_j^k)),\quad\forall i\in[m],
\end{aligned}
$$
and single-agent stochastic approximation (SA), whose update is $$x_{sg}^{k+1} = x_{sg}^{k} + \alpha(F_{i_k}(x_{sg}^k, \xi^k)-x_{sg}^k)$$
where $W = (w_{ij})_{i,j\in[m]}\in\mathbb{R}^{m\times m}$ is gossip matrix and $\xi_i^k, i_k, \xi^k$  are some random sampling. Then,  
\begin{proposition}[Decentralized speedup in a decomposable benchmark]\label{prop:linear-speedup}
Suppose  SA and DSA are run under the finite-sum setting above with the stepsizes $\alpha_1,\alpha_2$ and network parameter $\rho$ under some conditions. Then for any $k\geq 0$,

\begin{equation*}
\resizebox{\textwidth}{!}{$\displaystyle
\begin{aligned}
&(\mathrm{SA}):\quad
\mathbb{E}\|x_{sg}^{k+1}-x^\star\|_2^2
= \mathcal{O}\!\left((1-\alpha_1(1-\beta))^{2k}\right)\mathbb{E}\|x_{sg}^0-x^\star\|_2^2
+ \mathcal{O}\!\left(\frac{\alpha_1^2B^2}{1-(1-\alpha_1(1-\beta))^2}\right),\\
&(\mathrm{DSA})\!:
\max\left\{\!\|\bar x^k-x^\star\|^2, \frac{1}{m}\!\sum_{i = 1}^m \!\norm{x_i^k-\bar x^k}_2^2\!\right\}  \!=\! \mathcal{O}\!\left((1\!-\!\alpha_2(1\!-\!\beta)/2)^{2k}\right)
\!+\! \mathcal{O}\!\left(\!\frac{\alpha_2^2B^2}{m\bigl(1-(1-\alpha_2(1-\beta))^2\bigr)}\!\right),
\end{aligned}
$}
\end{equation*}
\end{proposition}

The details are in Appendix~\ref{sec:problem_set}. Proposition~\ref{prop:linear-speedup}  explains why decentralization is attractive: in a compatible decomposition, parallel sampling can shrink the steady-state stochastic error by a factor proportional to the number of agents. However, this benefit is meaningful only if the decentralized construction is faithful to the original fixed-point equation. Many large-scale problems including complex control tasks are not naturally decomposed as independent finite-sum components \cite{bellman1957dynamic,wei2019presslight,gleich2015pagerank}. Instead, the global vector is partitioned by coordinates, states, graph nodes, buses, or intersections, and the update for one owned block may depend on variables outside that block.

This is where naive decentralization can fail. A natural first attempt is a strict decomposition: each agent owns a disjoint block and evaluates updates using only information inside that block. While computationally appealing, strict decomposition is generally not exact for most fixed-point problems. In a Bellman equation, for example, the backup of an owned state-action pair may depend on successor states located in another region. If those cross-boundary values are truncated through strict decomposition, the resulting error is structural: it is caused by changing $\bar F$ itself, and therefore cannot be removed by more samples, smaller stepsizes, or additional consensus applied after the missing information has already been discarded.

Motivated by the above challenges, we introduce a Core-Halo decomposition. Intuitively, each agent owns a core set of coordinates that can only be updated by this agent itself, and is also assigned a halo set, which specifies the neighboring coordinates readable by it. Halo overlap is read-only: it supplies boundary context without creating duplicated ownership or conflicting writes. 
Based on Core-Halo decomposition, an operator implementable under decentralization is developed and it has exactly the same fixed points as in original problem. In words, Core-Halo separates two roles that strict decomposition conflates: who owns an update, and what information is needed to evaluate that update faithfully. Our main contributions are summarized as follows:
\begin{itemize}
    \item We identify a structural failure mode of strict local decomposition for decentralized fixed-point solving: when local updates require cross-boundary information, strict locality replaces the original mean operator by a biased surrogate.
    \item We propose the Core-Halo framework, which keeps update ownership disjoint while allowing overlapping read-only evaluation context, and we prove exact fixed-point recovery under an explicit locality condition.
    \item We validate the framework on extensive downstream scenarios, showing the wide applicability of the Core-Halo framework.
\end{itemize}

\section{Strict Local Decomposition}
Proposition~\ref{prop:linear-speedup} shows the benefit of decentralization in a favorable setting where the global operator is already given in a decomposition compatible with the agents. For a general fixed-point problem~\eqref{eq:problem_intro}, however, the operator \(\bar F\) is not necessarily presented in such a decentralized form. One must first decide how the global vector is partitioned and what information each agent may use to evaluate its local update.

A natural first attempt is a hard, nonoverlapping partition of the variables, in the spirit of classical parallel computation and nonoverlapping domain-decomposition schemes \cite{bertsekas1989parallel,smith1996domain,toselli2005domain}. Each agent owns one block, stores and samples only within that block, and evaluates its block update using only the information available inside it. This design is computationally attractive because it turns one large problem into smaller local problems with localized storage and sampling. Its limitation is that the locality is imposed by construction: if the true block update depends on variables outside the owned block, strict localization changes the mean operator itself. Any later communication or gossip step can average the quantities that were computed, but it cannot restore cross-boundary information that never entered the local update.

\begin{wrapfigure}{r}{0.3\linewidth}
  \centering
  \includegraphics[width=\linewidth]{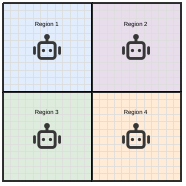}
  \caption{Strict Local Decomposition.}
  \label{fig:disjoint-regions}
  \vspace{-20pt}
\end{wrapfigure}

We first recall the notion of a partition, which will be used throughout the paper.
\begin{definition}
A collection of sets \(D_1, \cdots, D_m \subseteq \Omega\) is said to form a partition of \(\Omega\), if all \(D_i\)'s are pairwise disjoint and \(\cup_{i=1}^m D_i = \Omega\).
\end{definition}

Strict local decomposition formalizes the hard-partition rule at the level of the induced mean operator.
\begin{definition}[Strict  Decomposition]
Let \(D_1,\dots,D_m\) be a partition of \([d]\). We say that \(T^{\mathrm{hard}}:\mathbb{R}^d\to\mathbb{R}^d\) is a strict decomposition operator if there exist local maps \(G_i:\mathbb{R}^{|D_i|}\to\mathbb{R}^{|D_i|}\) such that
\[
(T^{\mathrm{hard}}(x))_{D_i}=G_i(x_{D_i}), \qquad i\in[m].
\]
\end{definition}
In words, agent \(i\) is allowed to use only its owned block \(x_{D_i}\) when producing an update on \(D_i\). Therefore, strict decomposition can be exact only when the true block \([\bar F(x)]_{D_i}\) is itself determined by \(x_{D_i}\). The next quantity measures precisely how much this condition fails.

\begin{proposition}[Bias of Strict Decomposition]
\label{thm:strict decomposition bias}
Fix \(i\in[m]\) and \(u\in\mathbb{R}^{|D_i|}\). Define
\[
\Delta_i(u)
:=
\sup\Bigl\{
\|(\bar F(x))_{D_i}-(\bar F(x'))_{D_i}\|_c
:
x_{D_i}=x'_{D_i}=u
\Bigr\}.
\]
Then every strict-decomposition operator \(T^{\mathrm{hard}}\) satisfies
\[
\sup_{x:\,x_{D_i}=u}
\|(T^{\mathrm{hard}}(x)-\bar F(x))_{D_i}\|_c
\geq
\frac{1}{2}\Delta_i(u).
\]
\end{proposition}
The proof is in Appendix~\ref{app:prop_sd_bias}. The quantity \(\Delta_i(u)\) is the diameter of the set of possible true block updates when the local block is fixed at \(u\) and only the outside coordinates vary. If \(\Delta_i(u)>0\), then a strict local map \(G_i(u)\) must assign one value to several possible true updates, so it must be wrong by at least half of this diameter on some instance. Proposition~\ref{thm:strict decomposition bias} therefore identifies a structural obstruction: once \([\bar F(x)]_{D_i}\) depends on coordinates outside \(D_i\), no block-local rule can uniformly recover the true mean operator on that block.

The obstruction is especially transparent for Bellman fixed points. Consider a finite discounted MDP \((\mathcal S,\mathcal A,P,r,\gamma)\), and let \(\{C_i\}_{i=1}^m\) be a partition of the state space such that \(\cup_{i\in[m]} C_i=\mathcal S\), the sets \(C_i\) are pairwise disjoint, and each \(C_i\) is connected. For Q-learning, the corresponding coordinate block is \(D_i=C_i\times\mathcal A\), and the original problem seeks the fixed point of
\[
(HQ)(s,a)
=
\sum_{s'\in\mathcal S}
P(s'| s,a)
\Bigl(r(s,a,s')+\gamma \max_{b\in\mathcal A}Q(s',b)\Bigr).
\]
An exact strict decomposition exists if and only if every component is dynamically closed:
\[
P(C_i| s,a)=1,
\qquad \forall s\in C_i,\ \forall a\in\mathcal A .
\]
Indeed, if the condition holds, then every Bellman backup starting from \(C_i\) uses only values in \(C_i\times\mathcal A\), so the block update is a function of \(Q_{D_i}\). Conversely, if there is a transition from some \(s\in C_i\) to a state outside \(C_i\) with positive probability, then changing the outside successor-state value changes \((HQ)(s,a)\) while leaving \(Q_{D_i}\) unchanged; no strict local map can represent both cases. Thus strict decomposition recovers the original Bellman problem only under the stringent closure condition above, which rules out all cross-component transitions under every action. Such closure is rarely satisfied in coupled MDPs, making strict decomposition too restrictive as a general decentralization principle. Details are given in Appendix~\ref{app:strict-decomposition-qlearning-gap}.

\section{Core-Halo Decomposition}\label{sec:Core-Halo}
The previous section shows that the limitation of strict decomposition is not merely a matter of consensus error, insufficient communication rounds, or noisy sampling. The failure is more fundamental: strict decomposition assumes that each owned block is closed under the mean operator. Equivalently, for every core $D_i$, the block $[\bar F(x)]_{D_i}$ must be determined by the local coordinates $x_{D_i}$ alone. Coupled fixed-point problems rarely have this property. More typically, the update on a local block depends on the block itself together with a small surrounding set of neighboring coordinates.

This motivates looking directly at the block-dependence structure of the global operator. For a candidate core $D_i\subseteq[d]$, call a set $S_i\supseteq D_i$ sufficient for evaluating the block $D_i$ if
\[
    x_{S_i}=x'_{S_i}
    \quad\Longrightarrow\quad
    [\bar F(x)]_{D_i}=[\bar F(x')]_{D_i},
    \qquad \forall x,x'\in\mathbb R^d .
\]
Equivalently, there exists a local map $T_i:\mathbb R^{|S_i|}\to\mathbb R^{|D_i|}$ such that
\[
    T_i(x_{S_i})=[\bar F(x)]_{D_i},
    \qquad \forall x\in\mathbb R^d .
\]
Thus $S_i$ is not another set of coordinates owned by agent $i$; it is the information required to evaluate the true update on $D_i$. Strict decomposition is viewed as a special case $S_i=D_i$, which is exact only when the block is already closed. The Core-Halo construction removes this unnecessary restriction.

\begin{definition}[Core-Halo]\label{def:Core-Halo}
For any index set $U\subseteq[d]$, let $x_U$ denote the restriction of $x$ to the coordinates in $U$. A family of pairs $\{(D_i,S_i)\}_{i=1}^m$, with $D_i,S_i\subseteq[d]$, is called a Core-Halo decomposition if $D_i\subseteq S_i\subseteq[d]$ for every $i$, and the cores $D_1,\ldots,D_m$ form a partition of $[d]$. Agent $i$ may read the variables in its halo $S_i$, but may write updates only on its core $D_i$. We say that the decomposition is compatible with $\bar F$ if each halo $S_i$ is sufficient for evaluating $[\bar F(x)]_{D_i}$ in the sense above.
\end{definition}

The Core-Halo framework is the key distinction from both centralized computation and strict local decomposition. The cores are disjoint, so every coordinate has a unique owner and no update is duplicated. The halos may overlap, so different agents may read the same boundary variables when those variables are needed for their local evaluations. Overlap is therefore purely informational. It supplies boundary context, but it does not create shared ownership or conflicting writes. Consequently, the decentralized implementation remains local while the local update can still match the block of the original global operator.

\begin{wrapfigure}{r}{0.38\linewidth}
  \centering\vspace{-6pt}
  \includegraphics[width=\linewidth]{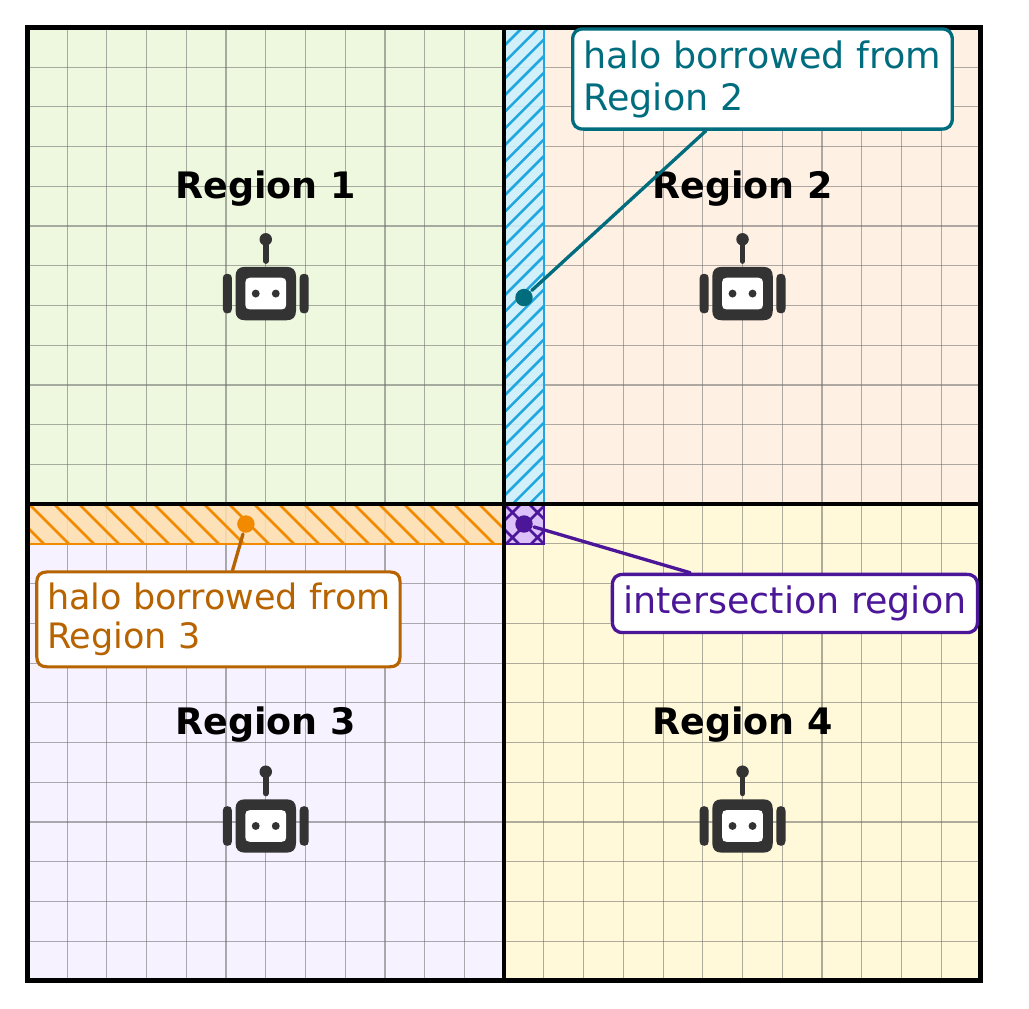}
  \caption{Core updates, halo for local context.}
  \label{fig:Core-Halo}
  \vspace{-30pt}
\end{wrapfigure}

We now state the structural payoff. Core-Halo does not require block closure under $\bar F$ and it only requires that every variable needed to evaluate the update on $D_i$ be present in the halo $S_i$. This condition is exactly aligned with the operator's intrinsic block-dependence structure.

\begin{theorem}[Exact recovery by Core-Halo]\label{lem:ch1}\label{thm:adv-ch}
Let $\{D_i\}_{i=1}^m$ be a partition of $[d]$. Suppose the Core-Halo decomposition $\{(D_i,S_i)\}_{i=1}^m$ is compatible with $\bar F$, so for each $i\in[m]$ there exists a map $T_i:\mathbb R^{|S_i|}\to\mathbb R^{|D_i|}$ such that $T_i(x_{S_i})=[\bar F(x)]_{D_i},
    \; \forall x\in\mathbb R^d $. Define the lifted Core-Halo operator $\widetilde T_i:\mathbb R^d\to\mathbb R^d$ by $ [\widetilde T_i(x)]_{D_i}:=T_i(x_{S_i}),
    \;
    [\widetilde T_i(x)]_{[d]\setminus D_i}:=x_{[d]\setminus D_i} $. Then 
    \begin{equation}
        \frac1m\sum_{i=1}^m \widetilde T_i(x)
    =
    \left(1-\frac1m\right)x+\frac1m\bar F(x),
    \; \forall x\in\mathbb R^d,
    \end{equation} and therefore $ x=\frac1m\sum_{i=1}^m \widetilde T_i(x)$ is equivalent to
    $x=\bar F(x)$.
\end{theorem}

The proof is in Appendix~\ref{subsec:thm1}. The theorem says that the averaged lifted Core-Halo operator is a relaxed version of the original mean operator, and hence has exactly the same fixed points, making the lifted formulation useful for general decentralized implementations. This establishes the formal separation from strict decomposition. Proposition~\ref{thm:strict decomposition bias} shows that if $\Delta_i(u)>0$ for some block and local configuration, then every strict local operator incurs nonzero blockwise bias. Core-Halo avoids this obstruction whenever the missing outside variables are included in the read-only halo: for all $i\in[m]$ and all $x\in\mathbb R^d$,
\[
    [\widetilde T_i(x)]_{D_i}=[\bar F(x)]_{D_i}.
\]
Thus outside coordinates need not be owned by agent $i$; they only need to be readable when they are part of the true block update.

{\bf Specialization of Bellman operator.}
A direct application is the Bellman fixed-point problem. Let $H$ denote the Bellman optimality operator of a finite discounted MDP. Let $\{C_i\}_{i=1}^m$ be a partition of the state space $\mathcal S$, and let $S_i\supseteq C_i$ be state halos satisfying the one-step successor condition
\[
    P(S_i\mid s,a):=\sum_{s'\in S_i}P(s'\mid s,a)=1,
    \qquad \forall s\in C_i,\ \forall a\in\mathcal A .
\]

\begin{corollary}[Exact recovery of the global Bellman equation under a Core-Halo decomposition]
\label{thm:exact-recovery-Core-Halo}
Let $H$ denote the Bellman operator. Assume $\{C_i\}_{i=1}^m$ is a partition of $\mathcal S$, and assume the state halos $S_i\supseteq C_i$ satisfy $P(S_i\mid s,a)=1$ for every $s\in C_i$ and $a\in\mathcal A$. For each agent $i$, define the local Bellman operator on its owned core by
\[
(H_iQ)(s,a)
:=
\sum_{s'\in S_i} P(s' | s,a)\bigl(r(s,a,s')+\gamma\max_{b\in\mathcal{A}}Q(s',b) \bigr),
\qquad (s,a)\in C_i\times\mathcal A,
\]
and the lifted operator $\widetilde H_i$ on $\mathbb R^{|\mathcal S||\mathcal A|}$ by
\[
(\widetilde H_iQ)(s,a)
:=
\begin{cases}
(H_iQ)(s,a), & s\in C_i,\\[1mm]
Q(s,a), & s\notin C_i.
\end{cases}
\]
Then the averaged operator $\hat H(Q):=\frac1m\sum_{i=1}^m \widetilde H_i(Q)$ satisfies
$\hat{H}(Q)=\left(1-\frac1m\right)Q+\frac1m H(Q)$ for any
$Q\in\mathbb R^{|\mathcal S||\mathcal A|}$, and therefore has exactly the same fixed point as the operator $H$.
\end{corollary}

The corollary highlights why Core-Halo is less restrictive than strict decomposition. Strict decomposition requires $P(C_i\mid s,a)=1$ for every state-action pair in $C_i$, which forbids all cross-core transitions. Core-Halo only requires $P(S_i\mid s,a)=1$: successors may leave the core, but they must remain inside the read-only halo. Hence the Bellman backup can use boundary information without assigning ownership of those boundary states to multiple agents.

\paragraph{Quantifying the strict-decomposition gap.}
We next use a deterministic gridworld to isolate the boundary effect in a setting where the geometry is simple enough to characterize exactly. The purpose of this example is not to restrict the phenomenon to grids, but to obtain a clean lower bound: square partitions create artificial boundaries whose number can be counted explicitly, and strict decomposition drops precisely the continuation values across those boundaries. Specifically, let \(Q^\star\) be the optimal action-value function satisfying \(Q^\star=HQ^\star\). For any candidate decentralized operator \(\mathcal K\), a necessary condition for preserving the original Bellman equation is \(\mathcal KQ^\star=Q^\star\). We therefore measure structural inconsistency by
\[
\operatorname{Dev}(\mathcal K)
:=
\frac{1}{4N}
\sum_{s\in \mathcal S}\sum_{a\in\mathcal A}
\left[
(\mathcal KQ^\star)(s,a)-Q^\star(s,a)
\right]^2 .
\]

Consider a deterministic \(n\times n\) gridworld with \(N=n^2\) states, four actions, discount \(\gamma\in(0,1)\), deterministic successor map \(f\), and one absorbing target \(t\) with \(f(t,a)=t\). We set \(Q(t,a)=0\) and write \(V_Q(s)=\max_{a\in\mathcal A}Q(s,a)\). Rewards are successor-state based: moving into \(t\) gives \(R_T>0\), moving into a trap \(z\in Z\subset\mathcal S\setminus\{t\}\) gives \(-R_Z(z)\), and all other transitions give zero. Thus
\[
(HQ)(s,a)=r(s,a,f(s,a))+\gamma V_Q(f(s,a)).
\]
Partition the grid into \(m=q^2>1\) equal square cores \(C_1,\ldots,C_m\), and let \(i(s)\) denote the core containing \(s\). The assembled strict-decomposition Bellman operator is
\[
(\mathcal H^{\rm hm}Q)(s,a)
=
r(s,a,f(s,a))
+
\gamma \mathbf 1\{f(s,a)\in C_{i(s)}\}V_Q(f(s,a)).
\]
Hence strict decomposition agrees with \(H\) on within-core transitions, but discards the continuation value whenever a transition crosses an artificial core boundary. Assume every non-target state has a trap-free path to \(t\) of length at most \(D_T\).

\begin{proposition}[Performance decay of strict-decomposition operator]
\label{prop:perform-sd}
Under the grid environment stated above,
\[
\operatorname{Dev}(\mathcal H^{\rm hm})
\ge
\frac{[\sqrt N(\sqrt m-1)-2]_+}{N}
\gamma^{2D_T}R_T^2,
\qquad
[x]_+:=\max\{x,0\}.
\]
In particular, for fixed \(N,\gamma,R_T,D_T\), the lower bound on the error induced by strict decomposition is nondecreasing in \(m\).
\end{proposition}

The proof is in Appendix~\ref{subsec:prop3}. Proposition~\ref{prop:perform-sd} makes the boundary mechanism explicit: increasing the number of square partitions increases the number of artificial boundary actions, and each useful cross-boundary action loses a continuation value that appears in the original Bellman backup. Notably, the same qualitative mechanism should occur on more general topologies whenever a partition cuts transitions whose successor states have nonzero value: strict decomposition removes those successor values, while a suitable halo keeps them as read-only context. 
\section{Empirical validation}
\label{sec: empirical validation}
In the experiment part, we would like to answer the following three Research Questions (RQs):

{\bfseries\itshape\ul{RQ1: Does strict decomposition decentralization introduce structural bias?}}
This question tests the failure mode identified by our theory. When a local update depends on variables outside the owned block, strict decomposition removes those dependencies from the target construction. We therefore examine whether strict local updates converge to, or perform like, a boundary-truncated surrogate rather than the original global problem.

{\bfseries\itshape\ul{RQ2: Can Core-Halo recover the original global fixed-point problem?}}
This question tests the main mechanism of Core-Halo. By keeping update ownership disjoint but allowing read-only access to boundary variables, Core-Halo should preserve the information needed to evaluate each local block of the original operator. We evaluate whether this added context removes the structural gap observed under strict decomposition.

{\bfseries\itshape\ul{RQ3: Is the Core-Halo advantage general across decentralized SA problems?}}
This question tests whether the benefit is tied to the proposed decomposition principle rather than to a single hand-crafted example. We therefore study problems with different fixed-point operators, dependency structures, and learning algorithms, including tabular Bellman updates, graph-linear PageRank, SARSA energy management, and DQN traffic control.

To answer \textbf{RQ3}, we evaluate the proposed framework across several decentralized learning and control settings, with additional scenarios discussed in Appendix~\ref{app:other}. Besides, all experiments were carried out on a workstation equipped with an NVIDIA GeForce RTX 4090 GPU (24GB VRAM), an Intel Core i9-13900K CPU.

\begin{wrapfigure}{r}{0.48\textwidth}
    \vspace{-40pt}
    \centering
    \includegraphics[width=0.47\textwidth]{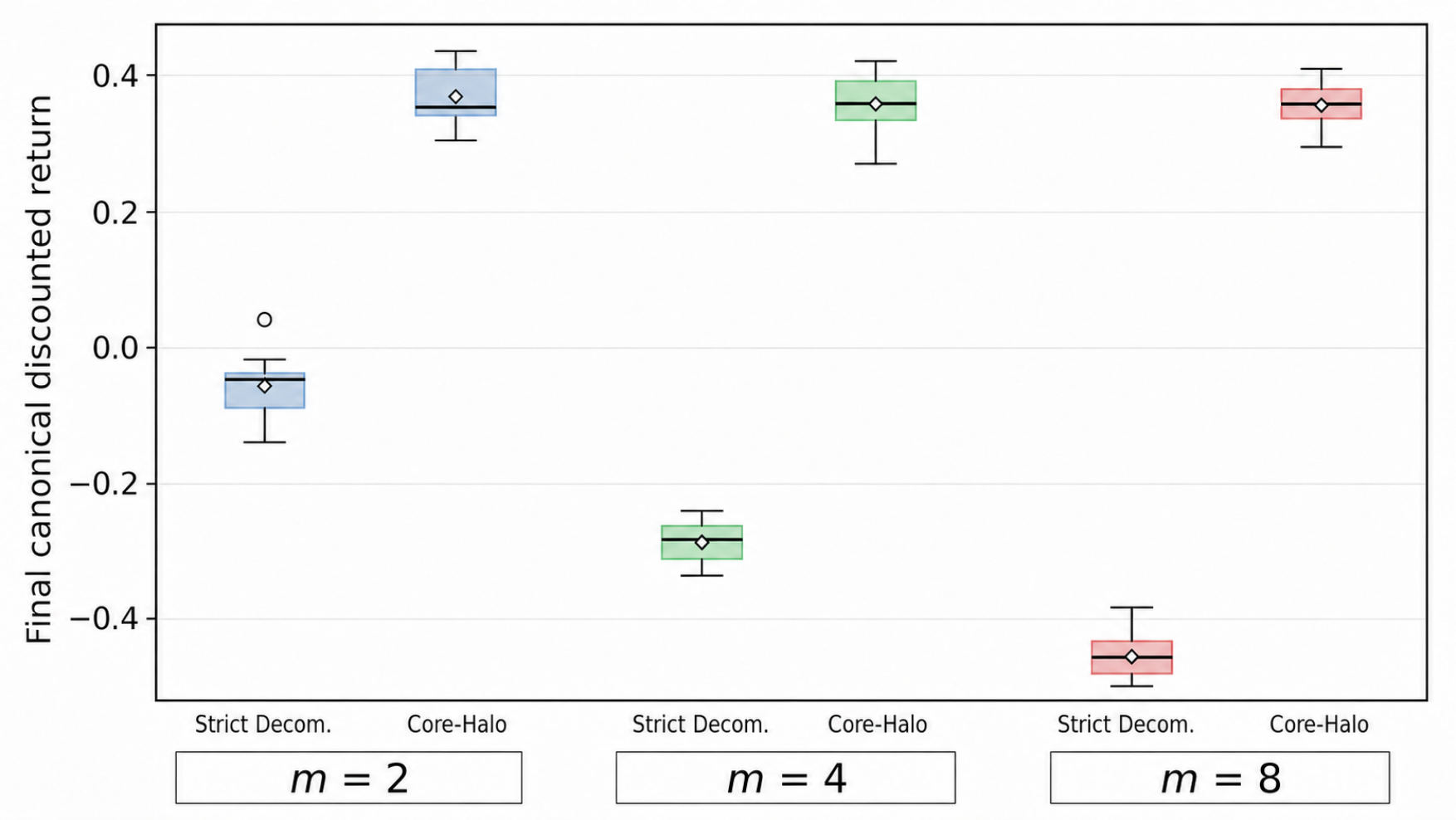}
    \caption{
    Return on the \(24\times 24\) MiniGrid.
    }
    \label{fig:minigrid-return}
    \vspace{-15pt}
\end{wrapfigure}
\subsection{Toy example: MiniGrid navigation}
\label{sec:exp-minigrid}
We first use a controlled MiniGrid navigation task \cite{chevalier2023minigrid} to illustrate the boundary effect predicted by our theory. The goal is to test (\textbf{RQ1}) whether strict local decomposition becomes less reliable as the number of partitions increases, and (\textbf{RQ2}) whether the proposed Core-Halo construction remains stable by preserving one-step cross-boundary evaluation context.

We construct a \(24 \times 24\) MiniGrid environment with symbolic tabular states. The state is the agent's grid position, and the action space consists
of local movement actions.  The environment contains a fixed goal, traps near
partition boundaries, and a per-step cost.  Boundary decisions are
important as the value of an action near the partition boundary may
depend on the value in a neighboring region.

For each \(m \in \{2,4,8\}\), we partition the grid into \(m\) rectangular
cores and assign one agent to each core.  Agents communicate only with spatially
adjacent regions using Metropolis weights.  We then compare the performance of strict decomposition and Core-Halo.  Thus ownership
remains disjoint, while the Bellman backup has access to the local context
needed at boundaries. All methods use the same decentralized tabular Q-learning protocol. Each agent
performs \(H=4\) local updates between communication rounds, and all
hyperparameters and total environment-step budgets are matched across methods.
Each configuration is run with \(10\) random seeds.  At evaluation time, we use
the averaged Q-function \(\bar Q^k = m^{-1}\sum_{i=1}^m Q_i^k\) and report the
canonical discounted return of the greedy policy under a fixed evaluation start
distribution.

Figure~\ref{fig:minigrid-return} shows the final canonical return.  As the
number of partitions increases, Strict Decomposition degrades substantially,
consistent with the fact that more partitions create more artificial boundaries
where continuation values are discarded.  In contrast, Core-Halo maintains similar performance across \(m=2,4,8\), and consistently outperforms strict decomposition at the same partition count.  This supports the central
claim that overlap is useful not as shared ownership, but as read-only context
for preserving correct Bellman backups near boundaries.

\subsection{PageRank}
We also implement the Core-Halo framework on personalized PageRank \cite{Page1998PageRank}. Given a row-stochastic matrix $P\in\mathbb{R}^{n\times n}$ and query node $u$, the PageRank vector satisfies
\[
x^\star=\alpha P^\top x^\star+(1-\alpha)e_u .
\]
With $\bar F(x):=\alpha P^\top x+(1-\alpha)e_u$, the map is an $\alpha$-contraction in $\ell^1$. We partition the nodes into disjoint cores $\{D_i\}_{i=1}^m$ and define the predecessor halo
\[
S_i:=D_i\cup\{\,k:\exists j\in D_i,\; P_{kj}>0\,\}.
\]
Then $[\bar F(x)]_{D_i}$ only depends on coordinates in $S_i$, so the locality condition in Definition \ref{def:Core-Halo} holds. Hence the lifted local operators satisfy $[\widetilde T_i(x)]_{D_i}=[\bar F(x)]_{D_i}$, and Theorem~\ref{lem:ch1} implies that the Core-Halo construction preserves the original PageRank fixed point. In the experiment, we construct a synthetic stochastic block model graph, row-normalize its adjacency matrix to obtain $P$, set $\alpha=0.85$, and compare three recursions from $x=0$: random single-agent block updates, parallel Core-Halo updates, and a strict decomposition that drops cross-core predecessors.
\begin{figure}[h]
  \centering
    \begin{subfigure}[t]{0.48\linewidth}
    \centering
    \includegraphics[height=4cm, width=6cm]{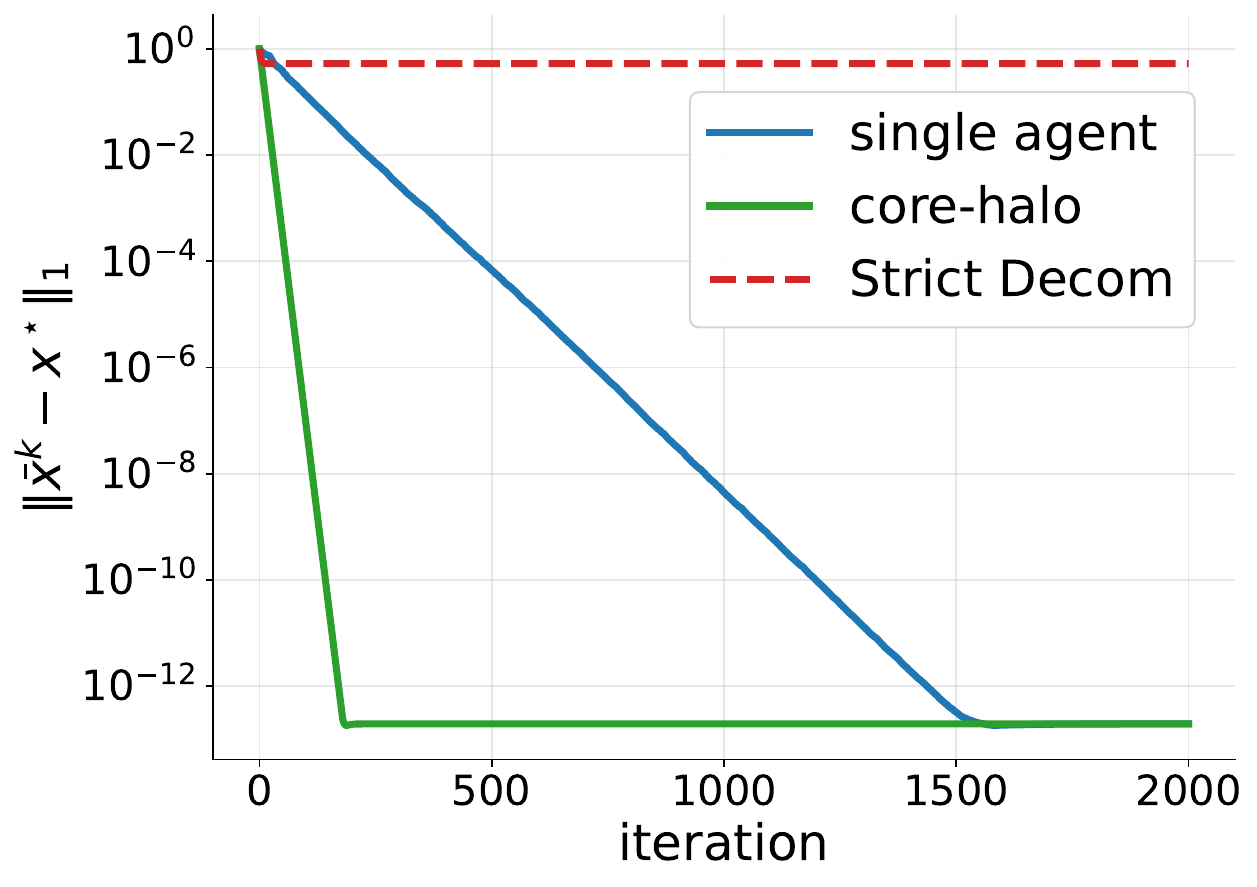}
    \caption{Convergence of the PageRank error $\|\bar x^k-x^\star\|_1$ over iterations.}
    \label{fig:pr_1}
  \end{subfigure}\hfill
  \begin{subfigure}[t]{0.48\linewidth}
    \centering
    \includegraphics[height=4cm, width=6cm]{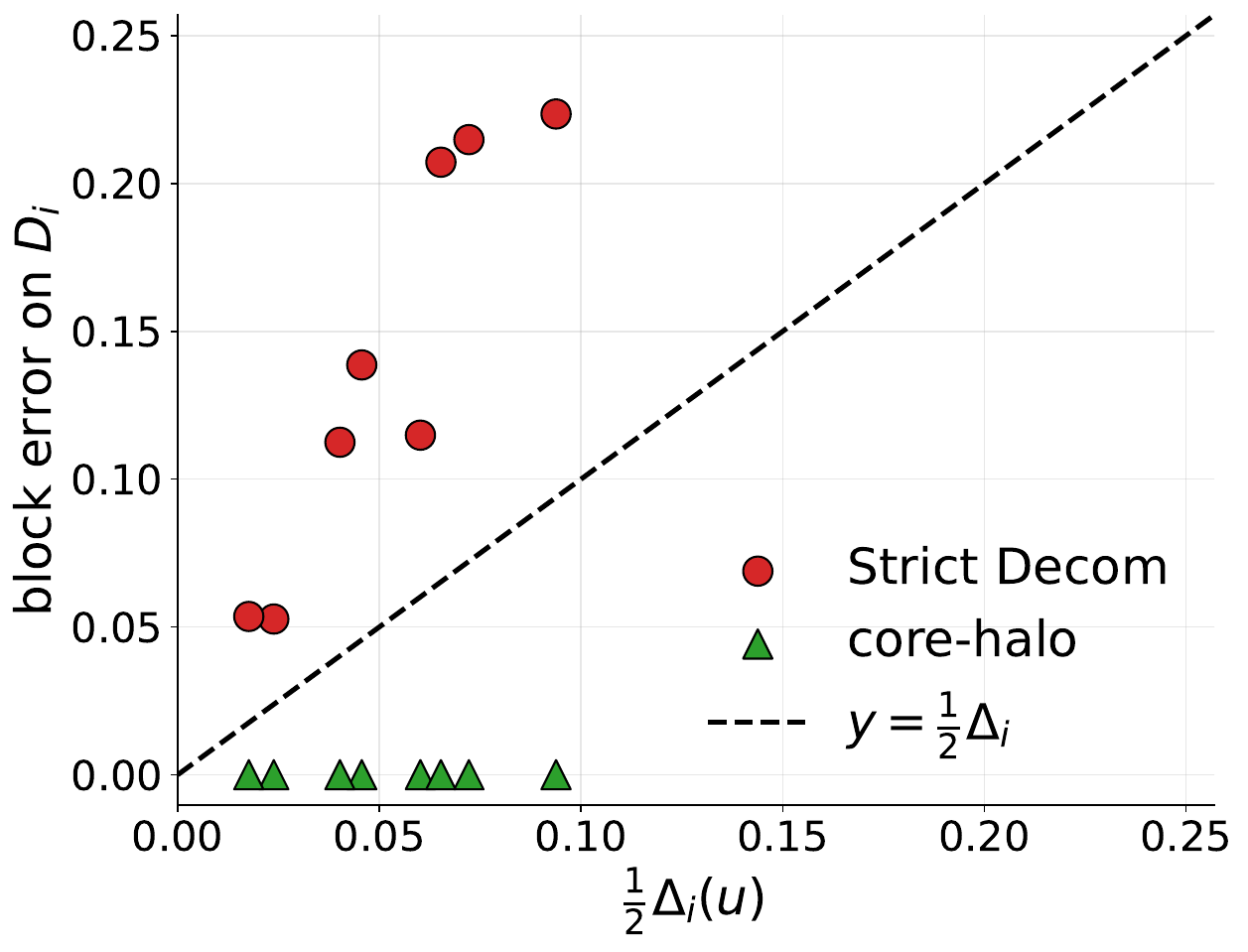}
    \caption{Empirical strict decomposition  block error versus the lower-bound scale $\tfrac12\Delta_i(u)$ across agents.}
    \label{fig:pr_2}
  \end{subfigure}
  \caption{Core-Halo preserves the original PageRank fixed point and converges rapidly, while strict decomposition incurs structural bias from dropping cross-core predecessor dependencies.}\vspace{-10pt}
\label{fig:pagerank}
\end{figure}

Figure~\ref{fig:pagerank} plots the PageRank error $\|\bar x^k-x^\star\|_1$ and the blockwise structural error across agents. Figure~\ref{fig:pr_1} shows that strict decomposition stays at a nonzero error plateau, indicating that dropping cross-core predecessor dependencies change the fixed point. This  answers \textbf{RQ1} by demonstrating the structural bias introduced by strict local truncation. In contrast, Core-Halo and single-agent both converge to the same numerical floor, showing that Core-Halo preserves the original PageRank fixed point, which answers \textbf{RQ2}. Besides, Core-Halo converges much faster because it updates all cores in parallel and behaves like centralized power iteration with rate about $\alpha=0.85$, while single-agent refreshes only one of $m$ blocks per round and contracts at the slower effective rate $1-\frac{1-\alpha}{m}$. Besides, Figure~\ref{fig:pr_2} further validates Proposition \ref{thm:strict decomposition bias} and Theorem \ref{thm:adv-ch}: on top of the nonzero plateau , the empirical strict-decomposition block errors lie above the scale $\tfrac12\Delta_i(u)$ across agents, illustrating the irreducible bias caused by strict local truncation. But Core-Halo can directly mitigate this issue.

\subsection{SARSA: Decentralized Energy Management in Smart Grids}
We evaluate Core-Halo on decentralized smart-grid energy management using tabular SARSA on the IEEE 9-, 14-, and 30-bus systems. Each bus $i$ is an aggregator controlling a battery with action $a_{i,t}\in[-1,1]$ and energy state $e_{i,t}\in[0,1]$, updated by
\[
e_{i,t+1}=\min\{\max(e_{i,t}+a_{i,t},0),1\}.
\]
The local cost combines electricity payment and a neighborhood overload penalty:
\[
c_{i,t}
=
p_{i,t}(a_{i,t}\overline E_i+\ell_{i,t})
+
\lambda\Big[
\sum_{j\in N(i)}(a_{j,t}\overline E_j+\ell_{j,t})-T_i
\Big]_+,
\]
where $N(i)$ denotes the grid-neighborhood of bus $i$. The reward is $r_t=-\sum_i c_{i,t}$, so bounded price and load profiles imply bounded rewards. For a fixed policy $\pi$, tabular SARSA is a stochastic-approximation realization of the Bellman fixed point $Q^\pi=H^\pi Q^\pi$.

This setting directly instantiates our framework because the update for bus $i$ is not determined by its own state alone: the overload term depends on neighboring buses through $N(i)$. Thus strict decomposition, which uses only the local state of bus $i$, replaces the original Bellman operator by a truncated local surrogate. Core-Halo keeps disjoint ownership of local $Q$-table updates, but augments the local state with read-only neighbor battery information, providing the boundary context needed to evaluate the neighborhood-dependent cost. We compare centralized SARSA, strict-decomposition SARSA, and Core-Halo SARSA. Details are deferred to Appendix~\ref{supp_sarsa}.
\begin{figure}[h]
\centering
\includegraphics[width=0.9\linewidth]{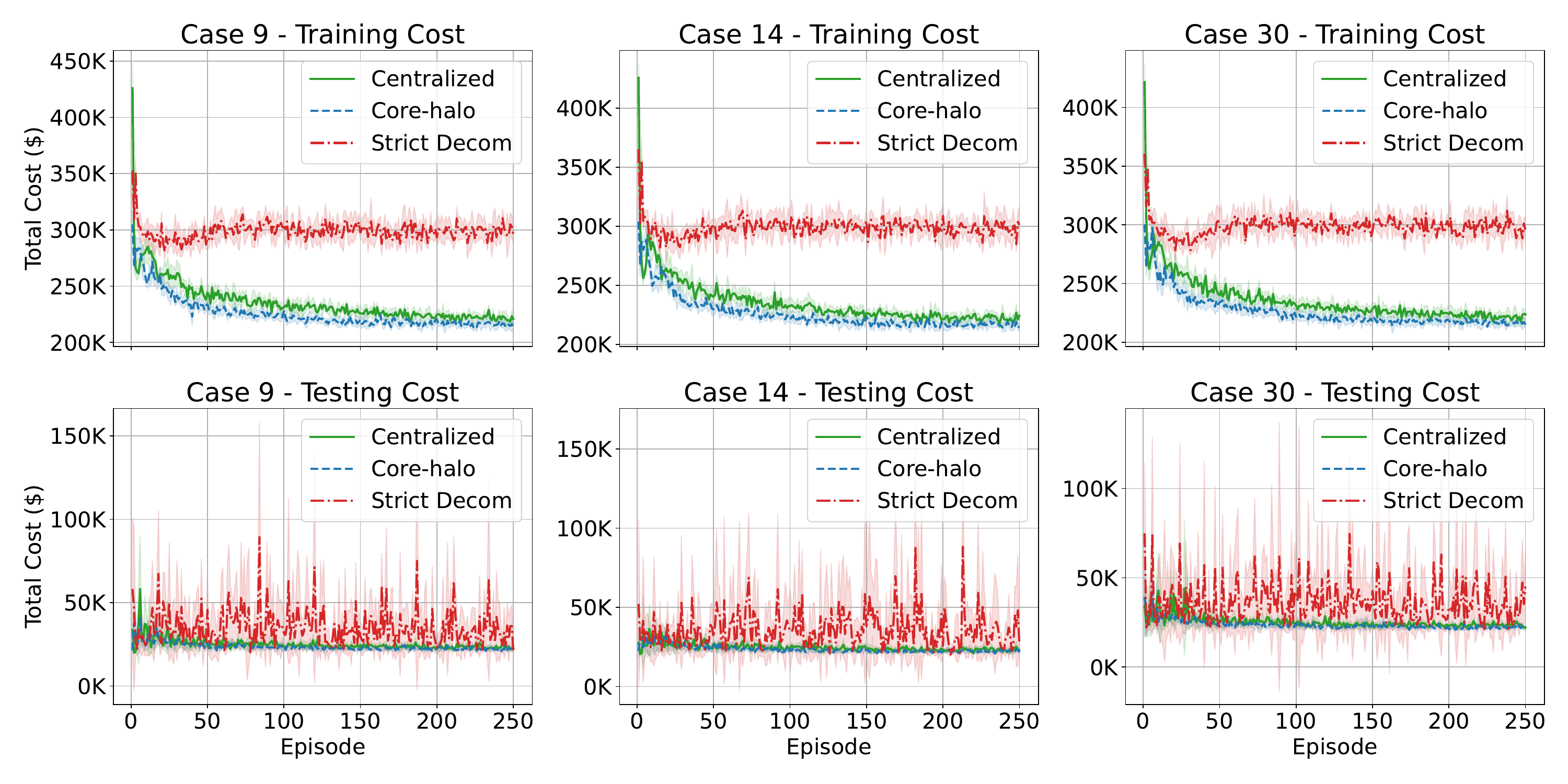}
\caption{Learning curves detailing the total grid cost trajectories across the IEEE 9, 14, and 30 bus.}
\label{fig:ieee-fig}\vspace{-10pt}
\end{figure}
The evaluation results shown in Figure~\ref{fig:ieee-fig} and summarized in Table~\ref{tab:ieee-eval} demonstrate a clear structural advantage for the Core-Halo decomposition. The independent strict-decomposition agents incur a strictly positive and irreducible performance gap, which answers \textbf{RQ1}. Due to learning with only self-observations, they truncate the cross-boundary dependencies required to anticipate and avoid the neighborhood demand penalty. Conversely, the Core-Halo framework successfully matches the optimal cost and achieves the same best rewards as the central planner. By augmenting the local update with neighborhood demand, the decentralized agents preserve the original global fixed-point problem without discarding essential boundary information, which answers \textbf{RQ2}.

\subsection{Deep Q-Network: Multi-Intersection Traffic Control}\label{sec:traffic-control} 

We evaluate Core-Halo on a multi-intersection traffic-control task using
DQN~\cite{mnih2015human,wei2019presslight,wei2019colight}. This experiment tests
\textbf{RQ1} and \textbf{RQ2} in a setting with nonlinear function approximation and
physical coupling across intersections. For a nonterminal transition, the DQN target is in the form of $y_t := r_t+\gamma \max_{a'\in\mathcal A} Q_{\theta^-}(s_{t+1},a')$, where \(Q_{\theta^-}\) is the target network. Conditioning on the full state-action pair,
\begin{equation}\label{eq:dqn-mean}
\mathbb E\!\left[y_t | s_t=s, a_t=a\right] = (H Q_{\theta^-})(s,a),
\end{equation}
with \(H\) denoting the Bellman optimality operator. The behavior policy affects the
sampling distribution of \((s_t,a_t)\), but not the conditional backup in \eqref{eq:dqn-mean}. Thus strict decomposition tests \textbf{RQ1} by forming local targets from only the owned intersections, which can discard cross-boundary traffic dependencies. Core-Halo tests \textbf{RQ2} by keeping the same disjoint update ownership while adding one-hop neighboring intersections as read-only context. Centralized DQN uses the full joint observation and serves as a full-information reference, not as a theoretical upper bound.
\begin{wraptable}{r}{0.5\linewidth}
\centering
\small\vspace{-10pt}
\caption{Comparison on the traffic environments.}
\label{tab:traffic-eval}
\resizebox{0.5\textwidth}{!}{
\begin{tabular}{cccc}
\toprule
\textbf{Env} & \textbf{Algorithm} 
& \textbf{Waiting (s)} ($\downarrow$) 
& \textbf{Throughput (veh/h)} ($\uparrow$) \\
\midrule
\multirow{3}{*}{Barrier ring}
& \cellcolor{gray!15}Centralized 
& \cellcolor{gray!15}15.78 $\pm$ 0.23
& \cellcolor{gray!15}1347.87 $\pm$ 1.51 \\
& Strict decom
& 16.35 $\pm$ 0.33 
& 1347.60 $\pm$ 2.12 \\
& Core-Halo
& 16.08 $\pm$ 0.23 
& 1346.27 $\pm$ 1.15 \\
\midrule
\multirow{3}{*}{DCB}
& \cellcolor{gray!15}Centralized 
& \cellcolor{gray!15}131.17 $\pm$ 2.10
& \cellcolor{gray!15}1057.07 $\pm$ 2.54 \\
& Strict decom
& 139.84 $\pm$ 11.65 
& 1045.33 $\pm$ 13.77 \\
& Core-Halo 
& 133.15 $\pm$ 2.16 
& 1055.07 $\pm$ 5.22 \\
\bottomrule
\end{tabular}\vspace{-20pt}
}
\end{wraptable}
We evaluate two SUMO topologies described in Appendix~\ref{supp_dqn}: a seven-intersection barrier ring partitioned into groups \([2,3,2]\), and a
fourteen-intersection Double-Corridor Bridge (DCB) with group structures \([2,3,2]\) and \([3,2,2]\). Vehicles enter and exit through traffic-light stubs under the same arrival process for all methods. Each method is trained with three random seeds and evaluated on five deterministic episodes per seed. We report mean waiting time and throughput. The one-hop halo should be interpreted as a practical approximation of Theorem~\ref{lem:ch1}. 

Figure~\ref{fig:training-curves} and Table~\ref{tab:traffic-eval} show that strict decomposition is less stable and gives worse waiting-time than the full-information reference on both topologies, supporting \textbf{RQ1}. The gap is especially large on the double-corridor bridge, where cross-boundary interactions are more pronounced. Core-Halo closes most of this gap using only one-hop read access, supporting the practical analogue of \textbf{RQ2}. Overall, the results support the paper's central message: in decentralized Bellman-target learning, separating update ownership from evaluation context can substantially improve performance even when the exact Core-Halo locality condition is only approximately satisfied.

\begin{figure}[ht]
  \centering
  \begin{subfigure}[t]{0.48\linewidth}
    \centering
    \includegraphics[height=4cm, width=7cm]{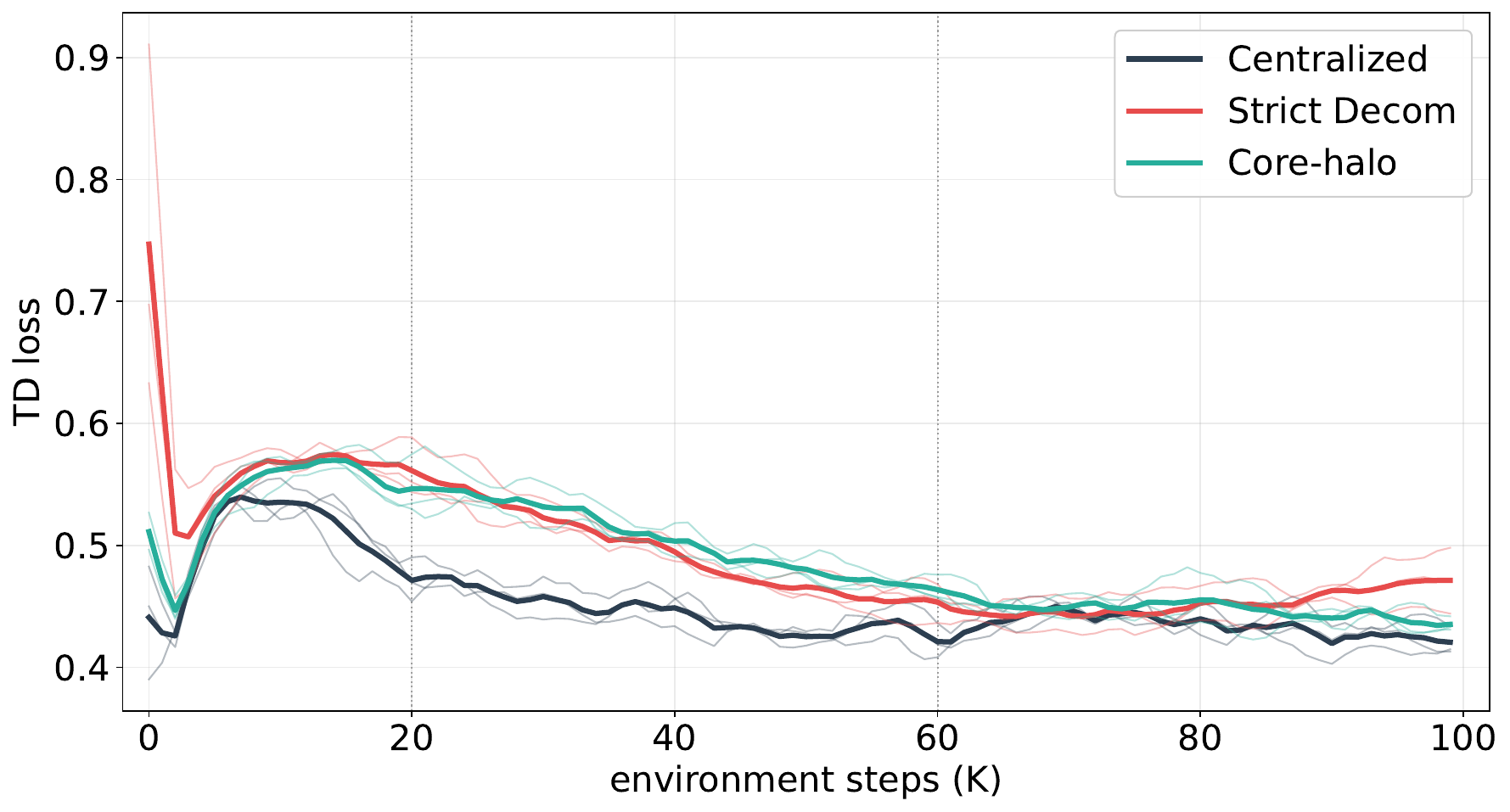}
    \caption{TD loss during training on the barrier-ring environment.}
    \label{fig:td-loss-gradient}
  \end{subfigure}\hfill
  \begin{subfigure}[t]{0.48\linewidth}
    \centering
    \includegraphics[height=4cm, width=7cm]{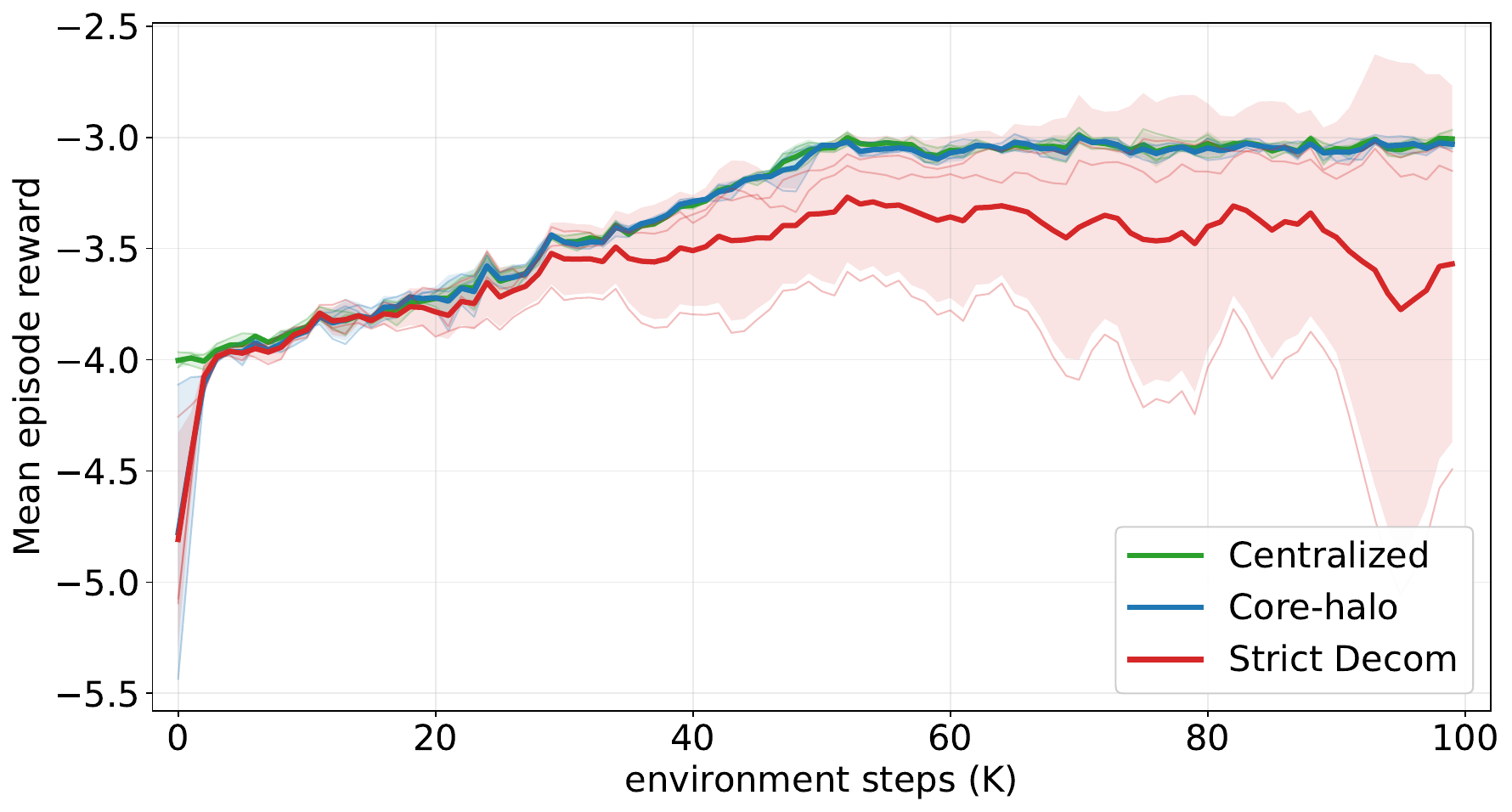}
    \caption{Mean episode reward during training on the
    double-corridor-bridge environment.}
    \label{fig:dcb-reward-multiseed}
  \end{subfigure}
  \caption{Traffic-Control Training Dynamics.}
  \label{fig:training-curves}
\end{figure}

\section{Conclusions}
We introduced Core-Halo decomposition for solving large-scale fixed-point problem. The central message is that decentralization must preserve the mean operator, not merely distribute computation: strict local decomposition can truncate cross-boundary dependencies and might converge to a biased surrogate. Core-Halo separates disjoint update ownership from read-only evaluation context, and under a locality condition exactly preserves the original fixed points. The bias bounds, Bellman gridworld analysis, and experiments on MiniGrid, PageRank, smart grids, and traffic control support this perspective. The potential limitations might include, more downstream applications and corresponding specific Core-Halo decompositions might need exploration, and in terms of theoretical aspect, there is lack of theoretical convergence analysis for DSA together with Core-Halo. Future work could also include adaptive halo selection, communication-privacy tradeoffs, and theory for approximate halos with function approximation.

\bibliographystyle{plain}
\bibliography{main}

\clearpage

\appendix

\section{Related Works}
\label{app:related-work}

\textbf{Stochastic approximation, Bellman fixed points, and implicit fixed-point models.}
The fixed-point viewpoint adopted here builds on the classical stochastic
approximation literature initiated by Robbins and Monro
\cite{robbins1951stochastic} and developed through the ODE and recursive
analyses of Kushner and Yin \cite{kushner2003stochastic} and Borkar
\cite{borkar2008stochastic}. In reinforcement learning, Bellman's dynamic
programming principle \cite{bellman1957dynamic} and the
stochastic-approximation interpretation of TD and value iteration give the
template for Bellman fixed-point algorithms; Watkins and Dayan
\cite{watkins1992q} introduced Q-learning, and Tsitsiklis
\cite{tsitsiklis1994asynchronous} together with Jaakkola, Jordan, and Singh
\cite{jaakkola1994convergence} established convergence for asynchronous
stochastic approximation and iterative dynamic programming. The same
mean-operator perspective covers TD learning \cite{sutton1988learning},
SARSA \cite{rummery1994line}, Expected SARSA \cite{vanseijen2009expected},
and finite-sample analyses of linear SA and TD learning
\cite{bhandari2018finite,srikant2019finite}. Beyond RL, PageRank is a
contractive graph fixed-point problem \cite{Page1998PageRank,gleich2015pagerank},
and implicit neural architectures such as deep equilibrium models, implicit
deep learning, monotone operator equilibrium networks, and Jacobian-free
implicit training also formulate learning as fixed-point solving
\cite{bai2019deq,ghaoui2021implicit,winston2020mondeq,fung2022jfb}.
Our work uses this fixed-point language but asks a different question: when
the global operator is implemented by decentralized local updates, does the
local construction preserve the original fixed-point equation, or has it
silently been replaced by a boundary-truncated surrogate?

\textbf{Decentralized stochastic approximation and multi-agent reinforcement learning.}
There are two relevant lines of prior work. The first is decentralized SA
itself: distributed fixed-point computation goes back to asynchronous
distributed methods \cite{bertsekas1983distributed} and the parallel-and-
distributed-computation framework of \cite{bertsekas1989parallel}, with early
analyses for communicating or consensus-coupled agents
\cite{kushner1987asymptotic,stankovic2011decentralized,stankovic2016distributed}.
The communication layer is typically a consensus or gossip protocol — fast
linear averaging \cite{xiao2004fast}, randomized gossip
\cite{boyd2006randomized}, distributed subgradient methods
\cite{nedic2009distributed}, ADMM \cite{boyd2011distributed}, or distributed
randomized PageRank \cite{ishii2010distributed}. In decentralized RL, \cite{kar2013qd} studied consensus-plus-innovation QD-learning,  \cite{zhang2018fully} analyzed fully decentralized learning with
networked agents, \cite{doan2019finite} and 
\cite{sun2020finite} gave finite-time analyses for distributed TD,  \cite{wang2020decentralized} developed decentralized TD tracking, \cite{zeng2023finite} established finite-time theory for
decentralized SA with fixed points. \cite{zhang2021decentralized} surveys
the broader area. The second line is the multi-agent RL benchmarks and
baselines used in our experiments: independent Q-learning \cite{IQL},
decentralized learning via local economic transactions
\cite{chang2020decentralized}, PettingZoo/MPE \cite{terry2021pettingzoo},
centralized-training/decentralized-execution methods such as MADDPG and
counterfactual policy gradients \cite{lowe2017multi,foerster2018counterfactual},
and value-factorization methods such as VDN and QMIX
\cite{sunehag2018value,rashid2018qmix}. Both lines take the local update
model as given and study communication, sample sharing, nonstationarity, or
training structure on top of it; our paper is complementary in that it
identifies a structural failure mode of strict decomposition itself and
gives a Core-Halo construction that preserves the intended global
fixed point.

\textbf{Locality, overlapping decompositions, sparse control, and boundary-aware evaluation.}
The distinction between a disjoint owned region and an overlapping context
region is related in spirit to domain decomposition and overlapping Schwarz
methods for numerical PDEs, where subproblems exchange boundary information
to recover the behavior of a larger coupled system
\cite{lions1988schwarz,smith1996domain,quarteroni1999domain,toselli2005domain}.
Similar locality ideas appear in sparse optimal control and MPC: splitting
methods exploit temporal and subsystem sparsity \cite{odonoghue2013splitting},
OSQP solves sparse QPs through operator splitting \cite{stellato2020osqp},
modern MPC texts formalize the fixed-point and optimization structure of
predictive control \cite{rawlings2017model}, and distributed or hierarchical
MPC reviews emphasize coupling constraints and neighboring subsystems
\cite{scattolini2009architectures,christofides2013distributed}; these ideas
underpin applications such as building and HVAC control \cite{serale2018mpc}.
In deep RL control, DQN made Bellman fixed-point learning practical with
function approximation \cite{mnih2015human}, and traffic-control systems
such as PressLight and CoLight show that multi-intersection policies depend
strongly on neighboring intersections \cite{wei2019presslight,wei2019colight}.
Our Core-Halo framework formalizes the overlap principle for decentralized
stochastic fixed-point problems: update ownership stays disjoint, evaluation
context may overlap, and under an explicit locality condition the resulting
decentralized recursion preserves the original global operator rather than
approximating an isolated local subproblem.

\section{Comparison between SA and DSA}

 Firstly, we provide details of Proposition~\ref{prop:linear-speedup}, in the spirit of theoretically comparing stochastic approximation (SA) and decentralized stochastic approximation (DSA). Then, we demonstrate the speedup of DSA over SA through numerical experiments in Appendix~\ref{sec:exp-dsa-speedup}. 

\subsection{Details of Proposition~\ref{prop:linear-speedup}}\label{sec:problem_set}

We consider $\bar F(x)$ with finite-sum structure:
\begin{equation}\label{eq:bar_F_sep}
\bar F(x) := \overm\summ \bar F_i(x), 
\end{equation}
with $\bar F_i(x) := \Eb_{\xi_i\sim \mu}[F_i(x, \xi_i)]$. 
Additionally, suppose: 
\begin{itemize}
    \item[(i)] There exists $L > 0$ such that for any $x, x^\prime \in \mathbb{R}^d$, $\xi \in \mathcal{X}$, $i\in[m]$ there is
$$
\max\{\norm{F_i(x, \xi) - F_i(x^\prime, \xi)}_2, \norm{F_i(x, \xi)-x - F_i(x^\prime, \xi)+x^\prime}_2\}\leq L\norm{x - x^\prime}_2.
$$    
    \item[(ii)]  There exists $\beta \in (0,1)$ such that for all $x, x^\prime \in\mathbb{R}^d$, $\norm{\bar F(x)  - \bar F(x^\prime)}_2 \leq \beta\norm{x-x^\prime}_2$.

    \item[(iii)] For any $\xi\in\mathcal{X}$, $\norm{F(x^\star, \xi)}_2 + \norm{x^\star}_2 \leq B$ for some $B > 0$.
\end{itemize}

Multiple agents can be applied to solve out the fixed point of \eqref{eq:bar_F_sep} over a connected and undirected network $\mathcal{G}$. The communication of the agents is conducted by gossip protocal \cite{ boyd2005gossip, boyd2006randomized} where each agent communicates only with its neighbors and updates by weighted averaging. Specifically, suppose  agent $i$ owns $x_i\in\mathbb{R}^d$, then one communication is by $X^{+} := WX$, where $X = [x_1, \cdots, x_m]^{\top} \in \mathbb{R}^{m\times d}$ and $W := (w_{ij})_{1\leq i,j\leq m}\in\mathbb{R}^{m\times m}$ is mixing matrix, satisfying:
\begin{itemize}
    \item[(i)] Each $w_{ij} \geq 0$, and $w_{ij} > 0$ if and only if $i$ and $j$ are adjacent in $\mathcal{G}$.
    \item[(ii)] $W^{\top} = W$,\quad $W\mathbf{1} = \mathbf{1}$, where $\mathbf{1}\in\mathbb{R}^m$ is all-one vector.
\end{itemize}
Denote $J := \frac{1}{m}\mathbf 1\mathbf 1^{\top}$ and  $\rho := \norm{W-J}_2$: the smaller $\rho$ is, the better quality the network is. 

Consider two algorithms,
 \begin{itemize}
     \item Stochastic Approximation (SA) with single agent: the index $i_k$ is drawn uniformly from $[m]$ and $\xi^k \sim\mu, x_{sg}^0 = 0$, 
\begin{equation}\label{example:sa_update}
    x_{sg}^{k+1} = x_{sg}^k + \alpha_1 (F_{i_k}(x_{sg}^k, \xi^k) - x_{sg}^k).
\end{equation}
\item Decentralized Stochastic Approximation (DSA) with network $\mathcal{G}$ and gossip matrix $W$: with i.i.d. sampling $\xi_i^k \sim \mu$, $x_i^0 = 0$,
\begin{equation}\label{example:dsa_update}
    \begin{aligned}
        & z_i^k = x_i^k + \alpha_2(F_i(x_i^k, \xi_i^k) - x_i^k),\quad\quad\forall i\in [m]\\
        & x_i^{k+1} = \sum_{j=1}^m w_{ij}z_j^k,\quad\quad\quad\forall i\in [m].
    \end{aligned}
\end{equation}
 \end{itemize}






We provide following key proposition on one-step recursive results for SA and DSA, under proper conditions on the network parameter $\rho$ and stepsizes $\alpha_1, \alpha_2 > 0$. By directly iterating these recursive inequalities from $k$ to $0$, we can derive Proposition~\ref{prop:linear-speedup} and omitted.

\begin{proposition}\label{prop:app_dsa_sa}
Choose $\alpha_1$ such that
$\alpha_1 \in \left(0, \min\left(
1,
\frac{2(1-\beta)}{(1-\beta)^2+2L^2},
\right)\right]$ 
and choose $\alpha_2 \in (0,1)$ such that  
\begin{equation}
\begin{aligned}
\left(1-\frac{\alpha_2(1-\beta)}{2}\right)^2
+
\frac{4\alpha_2^2L^2}{m}
 <1,\quad
\frac{4\alpha_2^2L^2}{3}
+
\frac{
2\alpha_2^3(1-\beta)\left(8-7\alpha_2(1-\beta)\right)L^2
}{
m\left(4-3\alpha_2(1-\beta)\right)^2
}
 <1.
\label{cond:alpha2-second}
\end{aligned}
\end{equation}
Assume further that the network parameter $\rho$ satisfies
\begin{align}
\rho^2
&\leq
\min\left(
\frac{
(1-\beta)^2
\left(8-7\alpha_2(1-\beta)\right)
\left(8-5\alpha_2(1-\beta)\right)
}{
64
\left(4-3\alpha_2(1-\beta)\right)^2
\left((L+\beta)^2+L^2\right)
},
\right.
\notag\\
&\qquad\qquad
\frac{
\alpha_2^2L^2
}{
6\left((1-\alpha_2+\alpha_2L)^2+\alpha_2^2L^2\right)
},
\left.
\frac{
\alpha_2(1-\beta)
\left(8-7\alpha_2(1-\beta)\right)
}{
8m\left(4-3\alpha_2(1-\beta)\right)^2
}
\right).
\label{cond:rho}
\end{align}
Run SA \eqref{example:sa_update} with stepsize $\alpha_1$ and DSA
\eqref{example:dsa_update} with stepsize $\alpha_2$. Then the generated
iterates $\{x_{sg}^k\}_k$ and $\{(x_i^k)_{i\in[m]}\}_k$ satisfy
\begin{equation}\label{ineq:x_sg}
\Eb\|x_{sg}^{k+1}-x^\star\|_2^2
\leq
\left((1-\alpha_1(1-\beta))^2+2\alpha_1^2L^2\right)
\Eb\norm{x_{sg}^k-x^\star}_2^2
+
2\alpha_1^2B^2.
\end{equation}
Moreover, with
\begin{align}
\tilde{\cL}^{k}
:=
\Eb\norm{\bar x^k-x^\star}_2^2
+
\frac{1}{m}\left(1+\frac{1}{\nu}\right)
\Eb\norm{(I-J)X^k}_F^2,
\end{align}
where  
$\nu
:=
\frac{
\alpha_2(1-\beta)\left(8-7\alpha_2(1-\beta)\right)
}{
16\left(1-\alpha_2(1-\beta)\right)^2
}$, then $\nu>0$ and  
\begin{equation}\label{ineq:tilde_L}
\begin{aligned}
&\tilde{\cL}^{k+1} \\
&\leq
\left(
\left(1-\frac{\alpha_2(1-\beta)}{2}\right)^2
+
\frac{4\alpha_2^2L^2}{m}
+
\frac{4\alpha_2^2L^2}{3}
+
\frac{
2\alpha_2^3(1-\beta)\left(8-7\alpha_2(1-\beta)\right)L^2
}{
m\left(4-3\alpha_2(1-\beta)\right)^2
}
\right)
\tilde{\cL}^k
+
\frac{5}{m}\alpha_2^2B^2.
\end{aligned}
\end{equation}
In particular, under above conditions for $\alpha_1, \alpha_2, \rho$, the coefficient in
\eqref{ineq:x_sg} is strictly smaller than one,  the coefficient before $\tilde{L}^k$ of \eqref{ineq:tilde_L} is also strictly smaller than
one.
\end{proposition}

\begin{proof}   
Denote $q_1:=1-\alpha_1(1-\beta)$ and $q_2:=1-\alpha_2(1-\beta)$ and define $\mathcal{F}_k := \sigma(\{(x_i^t)_{i\in[m]}, x_{sg}^t, (\xi_i^{t-1})_{i\in[m]}, \xi^{t-1}, i_{t-1}\}_{t=1}^k)$, $\Eb_k[.] := \Eb[.|\mathcal{F}_k]$. 

For the single-agent recursion, since $i_k$ is sampled uniformly from $[m]$,
\begin{equation}
    \Eb_k[F_{i_k}(x_{sg}^k,\xi^k)]
=
\frac{1}{m}\sum_{i=1}^m \bar F_i(x_{sg}^k)
=
\bar F(x_{sg}^k),
\end{equation}

Define
$ 
e^k
:=
F_{i_k}(x_{sg}^k,\xi^k)-\bar F(x_{sg}^k),
$ 
then $\Eb_k[e^k]=0$, and
\begin{equation}
x_{sg}^{k+1}
=
x_{sg}^k+\alpha_1(\bar F(x_{sg}^k)-x_{sg}^k)
+
\alpha_1 e^k.
\end{equation}
Since $x^\star=\bar F(x^\star)$, we have
\begin{equation}
x_{sg}^{k+1}-x^\star
=
(1-\alpha_1)(x_{sg}^k-x^\star)
+
\alpha_1(\bar F(x_{sg}^k)-\bar F(x^\star))
+
\alpha_1 e^k.
\end{equation}
Expanding the square gives
\begin{equation}
\begin{aligned}
&\|x_{sg}^{k+1}-x^\star\|_2^2
=
\left\|
(1-\alpha_1)(x_{sg}^k-x^\star)
+
\alpha_1(\bar F(x_{sg}^k)-\bar F(x^\star))
\right\|_2^2 \\
&\quad
+
2\alpha_1
\left\langle
(1-\alpha_1)(x_{sg}^k-x^\star)
+
\alpha_1(\bar F(x_{sg}^k)-\bar F(x^\star)),
 e^k
\right\rangle  
+
\alpha_1^2\|e^k\|_2^2.
\end{aligned}
\end{equation}
Taking conditional expectation and using $\Eb_k[e^k]=0$, the cross term vanishes:
\begin{align}
&\Eb_k\|x_{sg}^{k+1}-x^\star\|_2^2
=
\left\|
(1-\alpha_1)(x_{sg}^k-x^\star)
+
\alpha_1(\bar F(x_{sg}^k)-\bar F(x^\star))
\right\|_2^2 
+
\alpha_1^2\Eb_k\| e^k\|_2^2.
\end{align} 
Since $\bar F$ is a $\beta$-contraction,
\begin{equation}
\left\|
(1-\alpha_1)(x_{sg}^k-x^\star)
+
\alpha_1(\bar F(x_{sg}^k)-\bar F(x^\star))
\right\|_2
\leq
q_1\|x_{sg}^k-x^\star\|_2.
\end{equation}
Therefore,
\begin{equation}
\Eb_k\|x_{sg}^{k+1}-x^\star\|_2^2
\leq
q_1^2\|x_{sg}^k-x^\star\|_2^2
+
\alpha_1^2\Eb_k\|e^k\|_2^2.
\end{equation}
It remains to bound the stochastic term. Since
\begin{equation}
e^k
=
\big(F_{i_k}(x_{sg}^k,\xi^k)-x_{sg}^k\big)
-
\Eb_k\big[F_{i_k}(x_{sg}^k,\xi^k)-x_{sg}^k\big],
\end{equation}
we have
\begin{equation}
\Eb_k\|e^k\|_2^2
\leq
\Eb_k\|F_{i_k}(x_{sg}^k,\xi^k)-x_{sg}^k\|_2^2.
\end{equation}
Since for any $x\in\mathbb{R}^d, \xi\in\mathcal{X}$,
\begin{align*}
\|F_i(x,\xi)\!-\!x\|_2
&\!\leq\!
\|F_i(x,\xi)\!-\!x\!-\!F_i(x^\star,\xi)\!+\!x^\star\|_2
\!+\!
\|F_i(x^\star,\xi)\!-\!x^\star\|_2 \!\leq\!
L\|x\!-\!x^\star\|_2\!+\!B.
\end{align*}
Thus
\begin{equation}
\Eb_k\|e^k\|_2^2
\leq
\left(L\|x_{sg}^k-x^\star\|_2+B\right)^2.
\end{equation}
Combining the above estimates gives
\begin{equation}\label{eq:single-agent-recursion-tight}
\begin{aligned}
\Eb_k\|x_{sg}^{k+1}-x^\star\|_2^2
&\leq
q_1^2\|x_{sg}^k-x^\star\|_2^2
+
\alpha_1^2
\left(L\|x_{sg}^k-x^\star\|_2+B\right)^2\\
& \leq  (q_1^2+2\alpha_1^2 L^2)\norm{x_{sg}^k-x^\star}_2^2 + 2\alpha_1^2 B^2.
\end{aligned}
\end{equation}

We now analyze the decentralized recursion using the same argument. Denote
$
\bar x^k:=\frac{1}{m}\sum_{i=1}^m x_i^k,
$ we have
\begin{equation}
\bar x^{k+1}
=
\bar x^k
+
\alpha_2
\left(
\frac{1}{m}\sum_{i=1}^m F_i(x_i^k,\xi_i^k)
-
\bar x^k
\right).
\end{equation}
Adding and subtracting $\bar F(\bar x^k)$ gives
\begin{equation}
\begin{aligned}
\bar x^{k+1}
&=
\bar x^k+\alpha_2(\bar F(\bar x^k)-\bar x^k) 
+
\frac{\alpha_2}{m}\sum_{i=1}^m
\left(
\bar F_i(x_i^k)-\bar F_i(\bar x^k)
\right) 
+ \frac{\alpha_2}{m}\sum_{i=1}^m
\left(
F_i(x_i^k,\xi_i^k)-\bar F_i(x_i^k)
\right).
\end{aligned}
\end{equation}
Since $x^\star=\bar F(x^\star)$, we have
\begin{equation}
\begin{aligned}
\bar x^{k+1}-x^\star
&=
(1-\alpha_2)(\bar x^k-x^\star)
+
\alpha_2(\bar F(\bar x^k)-\bar F(x^\star)) \\
&\quad
+
\frac{\alpha_2}{m}\sum_{i=1}^m
\left(
\bar F_i(x_i^k)-\bar F_i(\bar x^k)
\right) \\
&\quad
+ \frac{\alpha_2}{m}\sum_{i=1}^m
\left(
F_i(x_i^k,\xi_i^k)-\bar F_i(x_i^k)
\right).
\end{aligned}
\end{equation}
Define
$ 
e_i^k:=F_i(x_i^k,\xi_i^k)-\bar F_i(x_i^k),
$ 
then $\Eb_k[e_i^k]=0$. Expanding the square gives
\begin{equation}
\begin{aligned}
&\|\bar x^{k+1}-x^\star\|_2^2 \\
&=
\left\|
(1-\alpha_2)(\bar x^k-x^\star)
+
\alpha_2(\bar F(\bar x^k)-\bar F(x^\star))
+
\frac{\alpha_2}{m}\sum_{i=1}^m
\left(
\bar F_i(x_i^k)-\bar F_i(\bar x^k)
\right)
\right\|_2^2 \\
&\quad
+
2\alpha_2
\left\langle
(1-\alpha_2)(\bar x^k-x^\star)
+
\alpha_2(\bar F(\bar x^k)-\bar F(x^\star))
+
\frac{\alpha_2}{m}\sum_{i=1}^m
\left(
\bar F_i(x_i^k)-\bar F_i(\bar x^k)
\right),
\frac{1}{m}\sum_{i=1}^m e_i^k
\right\rangle \\
&\quad
+
\alpha_2^2
\left\|
\frac{1}{m}\sum_{i=1}^m e_i^k
\right\|_2^2.
\end{aligned}
\end{equation}
Use
$
\Eb_k\left[
\frac{1}{m}\sum_{i=1}^m e_i^k
\right]=0 
$
to remove the cross term. Hence
\begin{equation}
\resizebox{\textwidth}{!}{$\displaystyle
\begin{aligned}
&\Eb_k\|\bar x^{k+1}-x^\star\|_2^2 \\
&=
\left\|
(1-\alpha_2)(\bar x^k-x^\star)
+
\alpha_2(\bar F(\bar x^k)-\bar F(x^\star))
+
\frac{\alpha_2}{m}\sum_{i=1}^m
\left(
\bar F_i(x_i^k)-\bar F_i(\bar x^k)
\right)
\right\|_2^2 
+
\alpha_2^2
\Eb_k
\left\|
\frac{1}{m}\sum_{i=1}^m e_i^k
\right\|_2^2.
\end{aligned}
$}
\end{equation}

There is
\begin{equation}
\left\|
(1-\alpha_2)(\bar x^k-x^\star)
+
\alpha_2(\bar F(\bar x^k)-\bar F(x^\star))
\right\|_2
\leq
q_2\|\bar x^k-x^\star\|_2.
\end{equation}
Moreover, 
\begin{equation}
\left\|
\frac{1}{m}\sum_{i=1}^m
\left(
\bar F_i(x_i^k)-\bar F_i(\bar x^k)
\right)
\right\|_2
\leq
\frac{L}{\sqrt{m}}\norm{(I-J)X^{k}}_F.
\end{equation}
Therefore, for any $\nu>0$, by Young's inequality,
\begin{equation}
\begin{aligned}
&\left\|
(1-\alpha_2)(\bar x^k-x^\star)
+
\alpha_2(\bar F(\bar x^k)-\bar F(x^\star))
+
\frac{\alpha_2}{m}\sum_{i=1}^m
\left(
\bar F_i(x_i^k)-\bar F_i(\bar x^k)
\right)
\right\|_2^2 \\
&\leq
(1+\nu)q_2^2\|\bar x^k-x^\star\|_2^2
+
\left(1+\frac{1}{\nu}\right)
\frac{\alpha_2^2 L^2}{m}\norm{(I-J)X^{k}}_F^2.
\end{aligned}
\end{equation}
It remains to bound the stochastic term. Since the samplings are independent across agents,
\begin{equation}
\Eb_k
\left\|
\frac{1}{m}\sum_{i=1}^m e_i^k
\right\|_2^2
=
\frac{1}{m^2}
\sum_{i=1}^m
\Eb_k\|e_i^k\|_2^2.
\end{equation}
Also, 
\begin{equation}
\Eb_k\|e_i^k\|_2^2
\leq
\Eb_k\|F_i(x_i^k,\xi_i^k)-x_i^k\|_2^2.
\end{equation}
\begin{align}
\|F_i(x,\xi)-x\|_2
&\leq
\|F_i(x,\xi)-x-F_i(x^\star,\xi)+x^\star\|_2
+
\|F_i(x^\star,\xi)-x^\star\|_2 \notag \\
&\leq
L\|x-x^\star\|_2+B.
\end{align}
Thus
\begin{equation}
\Eb_k\|e_i^k\|_2^2
\leq
\left(
L\|x_i^k-x^\star\|_2+B
\right)^2.
\end{equation}

Combining the above estimates gives
\begin{equation}\label{eq:multi-agent-average-recursion-tight}
\begin{aligned}
\Eb_k\|\bar x^{k+1}-x^\star\|_2^2
&\leq
(1+\nu)q_2^2\|\bar x^k-x^\star\|_2^2
+
\left(1+\frac{1}{\nu}\right)
\frac{\alpha_2^2 L^2}{m}\norm{(I-J)X^{k}}_F^2
\\
&\quad
+
\frac{\alpha_2^2}{m^2}
\sum_{i=1}^m
\left(
L\|x_i^k-x^\star\|_2+B
\right)^2.
\end{aligned}
\end{equation}

For each $i$, by the triangle inequality,
\begin{align}
\|x_i^k-x^\star\|_2 \leq
\|x_i^k-\bar x^k\|_2
+
\|\bar x^k-x^\star\|_2 .
\end{align}
 
Therefore,  
\begin{align}
\left(
L\|x_i^k-x^\star\|_2+B
\right)^2
\leq
2L^2\|x_i^k-\bar x^k\|_2^2
+
2\left(
L\|\bar x^k-x^\star\|_2+B
\right)^2 .
\end{align}
Summing over $i \in [m]$ gives
\begin{align}
\sum_{i=1}^m
\left(
L\|x_i^k-x^\star\|_2+B
\right)^2
\leq
2L^2\norm{(I-J)X^k}_F^2
+
2m
\left(
L\|\bar x^k-x^\star\|_2+B
\right)^2 .
\end{align}
Plugging this bound into
\eqref{eq:multi-agent-average-recursion-tight}, we obtain
\begin{align}
\Eb_k\|\bar x^{k+1}-x^\star\|_2^2
&\leq
(1+\nu)q_2^2\|\bar x^k-x^\star\|_2^2
+
\left[
\left(1+\frac{1}{\nu}\right)
\frac{\alpha_2^2 L^2}{m}
+
\frac{2\alpha_2^2 L^2}{m^2}
\right]
\norm{(I-J)X^{k}}_F^2
\notag\\
&\quad
+
\frac{2\alpha_2^2}{m}
\left(
L\|\bar x^k-x^\star\|_2+B
\right)^2 .
\end{align}
Finally, we obtain the fully quadratic bound
\begin{align}
\Eb_k\|\bar x^{k+1}-x^\star\|_2^2
&\leq
\left(
(1+\nu)q_2^2+\frac{4\alpha_2^2L^2}{m}
\right)
\|\bar x^k-x^\star\|_2^2
\notag\\
&\quad
+
\left[
\left(1+\frac{1}{\nu}\right)
\frac{\alpha_2^2 L^2}{m}
+
\frac{2\alpha_2^2 L^2}{m^2}
\right]
\norm{(I-J)X^{k}}_F^2
+
\frac{4\alpha_2^2B^2}{m}.
\end{align}

It remains to bound the consensus error $\norm{(I-J)X^k}_F^2$. We consider
\begin{equation}
\|(I-J)X^{k+1}\|_F^2
=
\|(W-J)Z^k\|_F^2
\leq
\rho^2\|(I-J)Z^k\|_F^2.
\end{equation}
By the definition of projection onto the consensus subspace, for any $y\in\mathbb{R}^d$,
\begin{equation}
\|(I-J)Z^k\|_F^2
\leq
\sum_{i=1}^m\|z_i^k-y\|_2^2.
\end{equation}
We choose
\begin{equation}
    y=
\bar x^k+\alpha_2(\bar F(\bar x^k)-\bar x^k).
\end{equation}

Since
\begin{equation}
z_i^k
=
x_i^k+\alpha_2(F_i(x_i^k,\xi_i^k)-x_i^k),
\end{equation}

we get
\begin{align}
z_i^k-y
&=
(1-\alpha_2)(x_i^k-\bar x^k)
+
\alpha_2(F_i(x_i^k,\xi_i^k)-\bar F(\bar x^k)) \notag \\
&=
(1-\alpha_2)(x_i^k-\bar x^k)
+
\alpha_2(\bar F_i(x_i^k)-\bar F_i(\bar x^k)) \notag \\
&\quad +
\alpha_2(\bar F_i(\bar x^k)-\bar F(\bar x^k))
+
\alpha_2 e_i^k
\end{align}
Then
\begin{align}
\Eb_k\|z_i^k-y\|_2^2
&=
\left\|
(1-\alpha_2)(x_i^k-\bar x^k)
+
\alpha_2(\bar F_i(x_i^k)-\bar F_i(\bar x^k))
+
\alpha_2(\bar F_i(\bar x^k)-\bar F(\bar x^k))
\right\|_2^2 
+
\alpha_2^2\Eb_k\|e_i^k\|_2^2.
\end{align}
 
\begin{equation}
    \|\bar F_i(x_i^k)-\bar F_i(\bar x^k)\|_2
\leq
L\|x_i^k-\bar x^k\|_2.
\end{equation}

Also, as shown above,
\begin{equation}
    \Eb_k\|e_i^k\|_2^2
\leq
\left(
L\|x_i^k-x^\star\|_2+B
\right)^2.
\end{equation}

Therefore,
\begin{align}
\Eb_k[\norm{(I-J)X^{k+1}}_F^2]
&\leq
\rho^2
\sum_{i=1}^m
\left(
(1-\alpha_2+\alpha_2 L)\|x_i^k-\bar x^k\|_2
+
\alpha_2\|\bar F_i(\bar x^k)-\bar F(\bar x^k)\|_2
\right)^2 \notag\\
&\quad
+
\rho^2\alpha_2^2
\sum_{i=1}^m
\left(
L\|x_i^k-x^\star\|_2+B
\right)^2 .
\label{eq:consensus-tight}
\end{align}
Use
\begin{equation}
\sum_{i=1}^m
\left(
L\|x_i^k-x^\star\|_2+B
\right)^2
\leq
2m\left(L\|\bar x^k-x^\star\|_2+B\right)^2
+
2L^2\norm{(I-J)X^k}_F^2,
\end{equation}

we obtain
\begin{align}
\Eb_k[\norm{(I-J)X^{k+1}}_F^2]
&\leq
2\rho^2\left((1-\alpha_2+\alpha_2 L)^2+\alpha_2^2L^2\right)\norm{(I-J)X^k}_F^2 \notag\\
& 
+
2\rho^2\alpha_2^2
\sum_{i=1}^m
\|\bar F_i(\bar x^k)-\bar F(\bar x^k)\|_2^2
+
2\rho^2\alpha_2^2m
\left(L\|\bar x^k-x^\star\|_2+B\right)^2.
\label{eq:consensus-closed-final}
\end{align}
Combining \eqref{eq:single-agent-recursion-tight} and \eqref{eq:consensus-closed-final} yields

Taking total expectation in \eqref{eq:single-agent-recursion-tight}, we obtain
\begin{align}
\Eb\|x_{sg}^{k+1}-x^\star\|_2^2
&\leq
\left(q_1^2+2\alpha_1^2L^2\right)
\Eb\norm{x_{sg}^k-x^\star}_2^2
+
2\alpha_1^2B^2 .
\label{eq:single-agent-total-recursion}
\end{align}

For decentralized setting, taking total expectation in the average recursion gives
\begin{align}
\Eb\|\bar x^{k+1}-x^\star\|_2^2
&\leq
\left(
(1+\nu)q_2^2+\frac{4\alpha_2^2L^2}{m}
\right)
\Eb\|\bar x^k-x^\star\|_2^2
\notag\\
&\quad
+
\left[
\left(1+\frac{1}{\nu}\right)
\frac{\alpha_2^2 L^2}{m}
+
\frac{2\alpha_2^2 L^2}{m^2}
\right]
\Eb\norm{(I-J)X^{k}}_F^2
+
\frac{4\alpha_2^2B^2}{m}.
\label{eq:multi-agent-average-total-recursion}
\end{align}

Further simplify the consensus recursion. From
\eqref{eq:consensus-closed-final}, taking total expectation yields
\begin{align}
\Eb[\norm{(I-J)X^{k+1}}_F^2]
&\leq
2\rho^2\left((1-\alpha_2+\alpha_2 L)^2+\alpha_2^2L^2\right)
\Eb\norm{(I-J)X^k}_F^2
\notag\\
&\quad
+
2\rho^2\alpha_2^2
\Eb\sum_{i=1}^m
\|\bar F_i(\bar x^k)-\bar F(\bar x^k)\|_2^2
\notag\\
&\quad
+
2\rho^2\alpha_2^2m
\Eb\left(L\|\bar x^k-x^\star\|_2+B\right)^2 .
\label{eq:consensus-total-before-final-bound}
\end{align}
Using
\begin{align}
\left(L\|\bar x^k-x^\star\|_2+B\right)^2
\leq
2L^2\|\bar x^k-x^\star\|_2^2+2B^2,
\end{align}
we get
\begin{align}
\Eb[\norm{(I-J)X^{k+1}}_F^2]
&\leq
2\rho^2\left((1-\alpha_2+\alpha_2 L)^2+\alpha_2^2L^2\right)
\Eb\norm{(I-J)X^k}_F^2
\notag\\
&\quad
+
2\rho^2\alpha_2^2
\Eb\sum_{i=1}^m
\|\bar F_i(\bar x^k)-\bar F(\bar x^k)\|_2^2
\notag\\
&\quad
+
4\rho^2\alpha_2^2mL^2
\Eb\|\bar x^k-x^\star\|_2^2
+
4\rho^2\alpha_2^2mB^2 .
\label{eq:consensus-total-recursion-with-heterogeneity}
\end{align}

To bound the second term $
2\rho^2\alpha_2^2
\Eb\sum_{i=1}^m
\|\bar F_i(\bar x^k)-\bar F(\bar x^k)\|_2^2$, we decompose, for each $i \in [m]$,
\begin{align}
\bar F_i(\bar x^k)-\bar F(\bar x^k)
&=
\bar F_i(\bar x^k)-\bar F_i(x^\star)
+
\bar F_i(x^\star)
-
\bar F(x^\star)
+
\bar F(x^\star)-\bar F(\bar x^k) \notag\\
&=
\bar F_i(\bar x^k)-\bar F_i(x^\star)
+
\bar F_i(x^\star)-x^\star
+
\bar F(x^\star)-\bar F(\bar x^k).
\end{align}
By the triangle inequality,
\begin{equation}
\|\bar F_i(\bar x^k)-\bar F(\bar x^k)\|_2
 \leq
\|\bar F_i(\bar x^k)-\bar F_i(x^\star)\|_2
+
\|\bar F_i(x^\star)-x^\star\|_2
+
\|\bar F(x^\star)-\bar F(\bar x^k)\|_2 .
\end{equation}

Then, combining 
\begin{equation}
\|\bar F_i(\bar x^k)\!-\!\bar F_i(x^\star)\|_2
\leq
L\|\bar x^k\!-\!x^\star\|_2,
\,\,
\|\bar F(x^\star)\!-\!\bar F(\bar x^k)\|_2
\leq
\beta\|\bar x^k\!-\!x^\star\|_2,\,\, \|\bar F_i(x^\star)\!-\!x^\star\|_2
\leq
B.
\end{equation}
   
gives 
\begin{align}
\|\bar F_i(\bar x^k)-\bar F(\bar x^k)\|_2^2
&\leq
2(L+\beta)^2\|\bar x^k-x^\star\|_2^2
+
2B^2.
\end{align}
 
Therefore,
\begin{align}
2\rho^2\alpha_2^2
\Eb \sum_{i=1}^m
\|\bar F_i(\bar x^k)-\bar F(\bar x^k)\|_2^2 
&\leq
4\rho^2\alpha_2^2m(L+\beta)^2
\Eb\|\bar x^k-x^\star\|_2^2
+
4\rho^2\alpha_2^2mB^2.
\label{eq:heterogeneity-term-final-bound}
\end{align}

Plugging \eqref{eq:heterogeneity-term-final-bound} into
\eqref{eq:consensus-total-recursion-with-heterogeneity}, we get
\begin{align}
\Eb[\norm{(I-J)X^{k+1}}_F^2]
&\leq
2\rho^2\left((1-\alpha_2+\alpha_2 L)^2+\alpha_2^2L^2\right)
\Eb\norm{(I-J)X^k}_F^2
\notag\\
&\quad
+
4\rho^2\alpha_2^2m
\left((L+\beta)^2+L^2\right)
\Eb\|\bar x^k-x^\star\|_2^2
+
8\rho^2\alpha_2^2mB^2.
\label{eq:consensus-total-recursion-closed}
\end{align}

Now combining \eqref{eq:multi-agent-average-total-recursion} and
\eqref{eq:consensus-total-recursion-closed}, define
\begin{align}
\tilde{\mathcal{L}}^k
:=
\Eb\|\bar x^k-x^\star\|_2^2
+
\frac{1}{m}\left(1+\frac{1}{\nu}\right)
\Eb\norm{(I-J)X^{k}}_F^2 .
\end{align}
Then
\begin{align}
\tilde{\mathcal{L}}^{k+1}
&=
\Eb\|\bar x^{k+1}-x^\star\|_2^2
+
\frac{1}{m}
\left(1+\frac{1}{\nu}\right)
\Eb\norm{(I-J)X^{k+1}}_F^2
\notag\\
&\leq
\left(
(1+\nu)q_2^2+\frac{4\alpha_2^2L^2}{m}
\right)
\Eb\|\bar x^k-x^\star\|_2^2
\notag\\
&\quad
+
\left(
\left(1+\frac{1}{\nu}\right)
\frac{\alpha_2^2 L^2}{m}
+
\frac{2\alpha_2^2 L^2}{m^2}
\right)
\Eb\norm{(I-J)X^{k}}_F^2
+
\frac{4\alpha_2^2B^2}{m}
\notag\\
&\quad
+
\frac{1}{m}
\left(1+\frac{1}{\nu}\right)
\Bigg(
2\rho^2\left((1-\alpha_2+\alpha_2 L)^2+\alpha_2^2L^2\right)
\Eb\norm{(I-J)X^k}_F^2
\notag\\
&\qquad\qquad
+
4\rho^2\alpha_2^2m
\left((L+\beta)^2+L^2\right)
\Eb\|\bar x^k-x^\star\|_2^2
+
8\rho^2\alpha_2^2mB^2
\Bigg)
\notag\\
&=
\left(
(1+\nu)q_2^2
+
\frac{4\alpha_2^2L^2}{m}
+
4\left(1+\frac{1}{\nu}\right)\rho^2\alpha_2^2
\left((L+\beta)^2+L^2\right)
\right)
\Eb\|\bar x^k-x^\star\|_2^2
\notag\\
&\quad
+
\left(
\left(1+\frac{1}{\nu}\right)
\frac{\alpha_2^2 L^2}{m}
+
\frac{2\alpha_2^2 L^2}{m^2}
+
\left(1+\frac{1}{\nu}\right)
\frac{2\rho^2}{m}
\left((1-\alpha_2+\alpha_2 L)^2+\alpha_2^2L^2\right)
\right)
\Eb\norm{(I-J)X^{k}}_F^2
\notag\\
&\quad
+
\left(
\frac{4}{m}
+
8\left(1+\frac{1}{\nu}\right)\rho^2
\right)
\alpha_2^2B^2 .
\end{align}
 
Then
\begin{align}
\tilde{\mathcal{L}}^{k+1}
&\leq
\left(
(1+\nu)q_2^2
+
\frac{4\alpha_2^2L^2}{m}
+
4\left(1+\frac{1}{\nu}\right)\rho^2\alpha_2^2
\left((L+\beta)^2+L^2\right)
\right)
\Eb\|\bar x^k-x^\star\|_2^2
\notag\\
&\quad
+
\left(
\alpha_2^2L^2
+
\frac{2\nu\alpha_2^2L^2}{(\nu+1)m}
+
2\rho^2
\left((1-\alpha_2+\alpha_2 L)^2+\alpha_2^2L^2\right)
\right)
\left(1+\frac{1}{\nu}\right)
\frac{1}{m}
\Eb\norm{(I-J)X^{k}}_F^2
\notag\\
&\quad
+
\left(
\frac{4}{m}
+
8\left(1+\frac{1}{\nu}\right)\rho^2
\right)
\alpha_2^2B^2 .
\end{align}
Therefore,
\begin{align}
\tilde{\mathcal{L}}^{k+1}
&\leq
\max\Bigg(
(1+\nu)q_2^2
+
\frac{4\alpha_2^2L^2}{m}
+
4\left(1+\frac{1}{\nu}\right)\rho^2\alpha_2^2
\left((L+\beta)^2+L^2\right),
\notag\\
&\qquad\qquad
\alpha_2^2L^2
+
\frac{2\nu\alpha_2^2L^2}{(\nu+1)m}
+
2\rho^2
\left((1-\alpha_2+\alpha_2 L)^2+\alpha_2^2L^2\right)
\Bigg)
\tilde{\mathcal{L}}^k
\notag\\
&\quad
+
\left(
\frac{4}{m}
+
8\left(1+\frac{1}{\nu}\right)\rho^2
\right)
\alpha_2^2B^2 .
\label{eq:multi-agent-final-recursion-tilde-L}
\end{align}

Now choose $\nu>0$ such that
\begin{align}
(1+\nu)q_2^2
=
\left(1-\frac{3\alpha_2(1-\beta)}{4}\right)^2,
\end{align}
we obtain
\begin{align}
\tilde{\mathcal{L}}^{k+1}
&\leq
\max\Bigg(
\left(1-\frac{3\alpha_2(1-\beta)}{4}\right)^2
+
\frac{4\alpha_2^2L^2}{m}
\notag\\
&\qquad\qquad
+
\frac{
4\rho^2\alpha_2
\left(4-3\alpha_2(1-\beta)\right)^2
}{
(1-\beta)\left(8-7\alpha_2(1-\beta)\right)
}
\left((L+\beta)^2+L^2\right),
\notag\\
&\qquad\qquad
\alpha_2^2L^2
+
\frac{
2\alpha_2^3(1-\beta)\left(8-7\alpha_2(1-\beta)\right)L^2
}{
m\left(4-3\alpha_2(1-\beta)\right)^2
}
\notag\\
&\qquad\qquad
+
2\rho^2
\left((1-\alpha_2+\alpha_2 L)^2+\alpha_2^2L^2\right)
\Bigg)
\tilde{\mathcal{L}}^k
\notag\\
&\quad
+
\left(
\frac{4}{m}
+
\frac{
8\rho^2
\left(4-3\alpha_2(1-\beta)\right)^2
}{
\alpha_2(1-\beta)\left(8-7\alpha_2(1-\beta)\right)
}
\right)
\alpha_2^2B^2 .
\label{eq:multi-agent-final-recursion-tilde-L-nu-chosen}
\end{align}

Suppose $\rho$ satisfies
\begin{equation}\label{example_ineq:rho}
\begin{aligned}
&    \frac{
4\rho^2\alpha_2
\left(4-3\alpha_2(1-\beta)\right)^2
}{
(1-\beta)\left(8-7\alpha_2(1-\beta)\right)
}
\left((L+\beta)^2+L^2\right) \leq \left(1-\frac{\alpha_2(1-\beta)}{2}\right)^2-\left(1-\frac{3\alpha_2(1-\beta)}{4}\right)^2, \\
& 2\rho^2
\left((1-\alpha_2+\alpha_2 L)^2+\alpha_2^2L^2\right)  \leq \frac{\alpha_2^2 L^2}{3},\\
& \frac{
8\rho^2
\left(4-3\alpha_2(1-\beta)\right)^2
}{
\alpha_2(1-\beta)\left(8-7\alpha_2(1-\beta)\right)
} \leq \frac{1}{m},
\end{aligned}
\end{equation}
 then the recursive becomes
\begin{equation}
\begin{aligned}
&\tilde{\mathcal{L}}^{k+1}\\
&\leq
\max\Bigg(
\left(1-\frac{ \alpha_2(1-\beta)}{2}\right)^2
+
\frac{4\alpha_2^2L^2}{m}, 
\frac{4\alpha_2^2L^2}{3}
+
\frac{
2\alpha_2^3(1-\beta)\left(8-7\alpha_2(1-\beta)\right)L^2
}{
m\left(4-3\alpha_2(1-\beta)\right)^2
} 
\Bigg)
\tilde{\mathcal{L}}^k
\notag 
+ 
\frac{5}{m}
\alpha_2^2B^2\\
& \leq \left(\left(1-\frac{ \alpha_2(1-\beta)}{2}\right)^2
+
\frac{4\alpha_2^2L^2}{m}+
\frac{4\alpha_2^2L^2}{3}
+
\frac{
2\alpha_2^3(1-\beta)\left(8-7\alpha_2(1-\beta)\right)L^2
}{
m\left(4-3\alpha_2(1-\beta)\right)^2
} \right)\tilde{\cL}^k+\frac{5}{m}\alpha_2^2B^2.\\
\end{aligned}
\end{equation}

For comparison, the single-agent recursion is
\begin{align}
\Eb\|x_{sg}^{k+1}-x^\star\|_2^2
&\leq
\left((1-\alpha_1(1-\beta))^2+2\alpha_1^2L^2\right)
\Eb\norm{x_{sg}^k-x^\star}_2^2
+
2\alpha_1^2B^2. 
\end{align}

If $\alpha_1, \alpha_2$ are small enough with order $\mathcal{O}\left(1/L\right)$, then 
\begin{equation}
\begin{aligned}
& \Eb\|x_{sg}^{k+1}-x^\star\|_2^2 = \mathcal{O}\left((1-\alpha_1(1-\beta))^{2k}\right)\Eb\norm{x_{sg}^0-x^\star}_2^2 + \mathcal{O}\left(\frac{\alpha_1^2 B^2}{1-(1-\alpha_1(1-\beta))^2}\right),\\
&\tilde{\cL}^{k+1} = \mathcal{O}\left((1-\alpha_2(1-\beta)/2)^{2k}\right)\tilde{\cL}^0 + \mathcal{O}\left(\frac{\alpha_2^2 B^2}{m(1-(1-\alpha_2(1-\beta))^2)}\right).
\end{aligned}
\end{equation}
 
\end{proof}

\subsection{Empirical Demonstration of Speedup under Decentralization}
\label{sec:exp-dsa-speedup}

We run a small experiment to check that decentralized SA is faster than single-agent SA in iteration count. The setup is deliberately simple so that the finite-sum assumption used in Proposition~\ref{prop:linear-speedup} holds exactly. We run $10$ seeds for $600{,}000$ parallel iterations. SA draws one sample per iteration while DSA draws $m=4$.

We take the experiment on a a $16\times 16$ gridworld. The valid state space is four $7\times 7$ cores separated by a wall strip (rows and columns $\{7,8\}$, zero-indexed).
The goal $(15,15)$ is absorbing. Transitions are stochastic: the intended
action succeeds with probability $0.8$ and slips left or right each with
probability $0.1$ and stay is deterministic. Hitting a wall or halo
cell leaves the agent in place. Reward is $+1$ on entering the goal and $0$
otherwise discount $\gamma = 0.99$.

The policy $\pi$ is uniform over the five actions
$\{\texttt{up}, \texttt{right}, \texttt{down}, \texttt{left}, \texttt{stay}\}$.
The policy-evaluation Bellman operator
\[
(T_\pi V)(s) = \sum_a \pi(a\mid s) \sum_{s'} P(s'\mid s,a)\bigl[r(s,a,s') + \gamma V(s')\bigr]
\]
is a $\gamma$-contraction in $\|\cdot\|_\infty$. We compute the reference
$V^\pi$ by exact policy iteration to sup-norm tolerance $10^{-12}$. Since
the cores are wall-separated, every successor of a state in core $i$ stays
in core $i$, so $T_\pi$ splits as $T_\pi = \tfrac{1}{m}\sum_{i=1}^m T_{\pi,i}$
with $m=4$ and each $T_{\pi,i}$ supported on its own core. This matches
the finite-sum setting in Section~\ref{sec:problem_set}. Results are in Table~\ref{tab:exp-speedup-results} and Figures~\ref{fig:exp-speedup-curves}. The two methods reach roughly the same plateau, confirming that the matched-
plateau calibration worked. DSA hits the common $L_\infty$ band by about
$85{,}000$ iterations while SA needs about $470{,}000$, a $5.54\times$
gap; the $L_2$ gap is similar at $5.84\times$. Per-seed variance is also smaller for
DSA, which is consistent with the $1/m$ factor on the stochastic term in
Proposition~\ref{prop:linear-speedup}.

\begin{figure}[!ht]
\centering
\begin{minipage}[c]{0.50\linewidth}
  \centering
  \captionof{table}{DSA vs.\ SA on the wall-separated $16\times 16$ grid, $10$ seeds.}
  \label{tab:exp-speedup-results}
  \small
  \resizebox{\linewidth}{!}{%
  \begin{tabular}{ccccc}
    \toprule
    Method & $\alpha$ & $m$ & Stable hit ($L_\infty$) & Stable hit ($L_2$) \\
    \midrule
    SA  & $0.05$ & $1$ & $469{,}400 \pm 12{,}869$ & $521{,}111$ \\
    DSA & $0.30$ & $4$ & $\phantom{0}84{,}800 \pm \phantom{0}2{,}585$ & $\phantom{0}89{,}200$ \\
    \midrule
    \textbf{iteration speedup} & & & $\mathbf{5.54\times}$ & $\mathbf{5.84\times}$ \\
    \bottomrule
  \end{tabular}}
\end{minipage}%
\hfill
\begin{minipage}[c]{0.46\linewidth}
  \centering
  \captionof{figure}{$\|V^k - V^\pi\|_\infty$ versus parallel iterations.}
  \label{fig:exp-speedup-curves}
  \includegraphics[width=\linewidth]{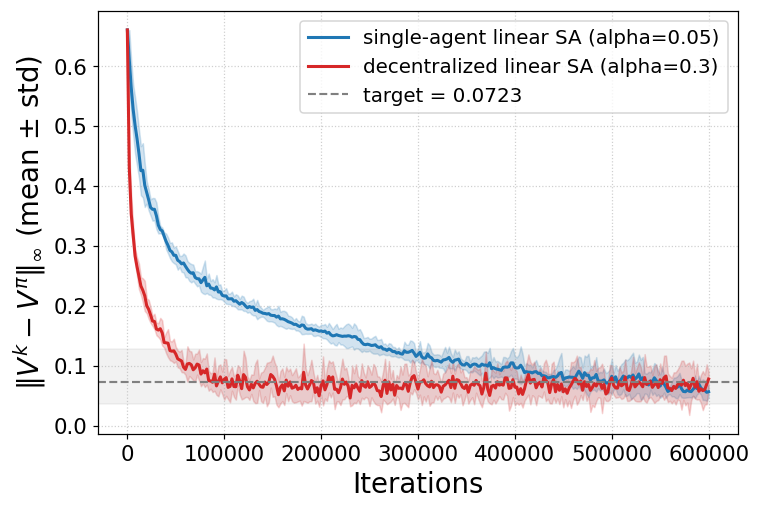}
\end{minipage}
\end{figure}

\section{Strict-decomposition operator relies on block closure}

\subsection{Proof of Proposition~\ref{thm:strict decomposition bias}}\label{app:prop_sd_bias}

\textbf{Proposition 2} (Bias of Strict Decomposition)
\textit{Fix $i\in[m]$ and $u\in\mathbb{R}^{|D_i|}$. Define
$$
\Delta_i(u)
:=
\sup\Bigl\{
\|(\bar F(x))_{D_i}-(\bar F(x'))_{D_i}\|_c
:
x_{D_i}=x'_{D_i}=u
\Bigr\}.
$$
Then every strict-decomposition  operator $T^{\mathrm{hard}}$ satisfies
$$
\sup_{x:\,x_{D_i}=u}
\|(T^{\mathrm{hard}}(x)-\bar F(x))_{D_i}\|_c
\geq
\frac{1}{2}\Delta_i(u).
$$}
\begin{proof}
Fix $i$ and $u$, and let $(T^{\mathrm{hard}}(x))_{D_i}=G_i(x_{D_i})$. If $x_{D_i}=x'_{D_i}=u$, then
$$
(T^{\mathrm{hard}}(x))_{D_i}
=
(T^{\mathrm{hard}}(x'))_{D_i}
=
G_i(u).
$$
Hence
$$
\|(\bar F(x))_{D_i}-(\bar F(x'))_{D_i}\|_c
\leq
\|(\bar F(x))_{D_i}-G_i(u)\|_c
+
\|G_i(u)-(\bar F(x'))_{D_i}\|_c.
$$
Therefore,
$$
\|(\bar F(x))_{D_i}-(\bar F(x'))_{D_i}\|_c
\leq
2\sup_{z:\,z_{D_i}=u}
\|(T^{\mathrm{hard}}(z)-\bar F(z))_{D_i}\|_c.
$$
Taking the supremum over all $x,x'$ such that $x_{D_i}=x'_{D_i}=u$ gives
$$
\Delta_i(u)
\leq
2\sup_{z:\,z_{D_i}=u}
\|(T^{\mathrm{hard}}(z)-\bar F(z))_{D_i}\|_c,
$$
which is exactly the claimed inequality.
\end{proof}

\subsection{Strict-decomposition MDP}
\label{app:strict-decomposition-qlearning-gap}

In this section, we consider a corner case that make strict decomposition works the same as centralized case.
 
\begin{proposition}
Consider a finite discounted MDP $(\mathcal S,\mathcal A,P,r,\gamma)$ with $0<\gamma<1$. Let $\{C_i\}_{i=1}^m$ be a partition of $\mathcal S$ such that $\cup_{i=1}^m C_i = \mathcal{S}$, each $C_i$ is connected and $C_i$'s are pairwise disjoint. Set
$D_i=C_i\times\mathcal A$, and the Bellman optimality operator
\begin{equation}
(HQ)(s,a)
=
\sum_{s'\in\mathcal S}
P(s'|s,a)
\Bigl(r(s,a,s')+\gamma \max_{b\in\mathcal A} Q(s',b)\Bigr)
\end{equation}
admits an exact strict decomposition with respect to the blocks $\{D_i\}$, meaning
that there exist maps $G_i:\mathbb R^{|D_i|}\to \mathbb R^{|D_i|}$ such that $(HQ)_{D_i}=G_i(Q_{D_i})$ for every $Q\in \mathbb R^{\mathcal S\times \mathcal A}$, if and only if every component is dynamically closed, $P(C_i|s,a)=1$ for all $i\in[m],\ \forall s\in C_i,\ \forall a\in\mathcal A$.
\end{proposition}

\begin{proof}
First suppose that each $C_i$ is dynamically closed, that is, for every
$s\in C_i$ and $a\in\mathcal A$, we have $P(s'|s,a)=0$ whenever
$s'\notin C_i$. Hence
\begin{equation}
(HQ)(s,a)
=
\sum_{s'\in C_i}
P(s'|s,a)
\Bigl(r(s,a,s')+\gamma \max_{b\in\mathcal A} Q(s',b)\Bigr).
\end{equation}

The right-hand side depends on $Q$ only through the entries
$\{Q(s',b):s'\in C_i,\ b\in\mathcal A\}=Q_{D_i}$. Therefore $H$ is strictly
decomposable by defining
\begin{equation}    
\bigl(G_i(q)\bigr)(s,a)
=
\sum_{s'\in C_i}
P(s'|s,a)
\Bigl(r(s,a,s')+\gamma \max_{b\in\mathcal A} q(s',b)\Bigr),
\qquad s\in C_i,\ a\in\mathcal A .
\end{equation}

Conversely, suppose that $H$ admits such a strict decomposition. Assume, toward
a contradiction, that some component is not dynamically closed. Then there exist
$i$, $s\in C_i$, $a\in\mathcal A$, and $y\notin C_i$ such that $p:=P(y|s,a)>0$. Choose an action $b_0\in\mathcal A$ and a scalar $M>0$. Define two action-value
functions $Q,Q'\in\mathbb R^{\mathcal S\times\mathcal A}$ by $Q(z,b)=0$ for all $(z,b)$,
and
\begin{equation}
Q'(z,b)=
\begin{cases}
M, & (z,b)=(y,b_0),\\
0, & \text{otherwise}.
\end{cases}
\end{equation}

Since $y\notin C_i$, the two functions agree on $D_i=C_i\times\mathcal A$: $Q_{D_i}=Q'_{D_i}$. Strict decomposability would therefore imply $(HQ)_{D_i}=G_i(Q_{D_i})=G_i(Q'_{D_i})=(HQ')_{D_i}$. In particular, it would imply $(HQ)(s,a)=(HQ')(s,a)$. But the reward terms cancel when subtracting, and the two functions differ only at state $y$, so
\begin{equation}
(HQ')(s,a)-(HQ)(s,a)
=
\gamma P(y|s,a)
\left(
\max_{b\in\mathcal A}Q'(y,b)-\max_{b\in\mathcal A}Q(y,b)
\right)
=
\gamma p M
>0.
\end{equation}

This contradicts $(HQ)(s,a)=(HQ')(s,a)$. Hence no such transition to
$\mathcal S\setminus C_i$ can exist. Since $i$ was arbitrary, every block
$C_i$ must be dynamically closed.
\end{proof}

\section{Analysis in Section~\ref{sec:Core-Halo}}\label{app:analysis_ch}

\subsection{Proof of Theorem~\ref{lem:ch1}}\label{subsec:thm1}

\textbf{Theorem 1} (Advantage of Core-Halo) \textit{Let $\{D_i\}_{i=1}^m$ be a partition of $[d]$. Suppose the Core-Halo decomposition $\{(D_i,S_i)\}_{i=1}^m$ is compatible with $\bar F$, so for each $i\in[m]$ there exists a map $T_i:\mathbb R^{|S_i|}\to\mathbb R^{|D_i|}$ such that $T_i(x_{S_i})=[\bar F(x)]_{D_i},
    \; \forall x\in\mathbb R^d $. Define the lifted Core-Halo operator $\widetilde T_i:\mathbb R^d\to\mathbb R^d$ by $ [\widetilde T_i(x)]_{D_i}:=T_i(x_{S_i}),
    \;
    [\widetilde T_i(x)]_{[d]\setminus D_i}:=x_{[d]\setminus D_i} $. Then 
    \begin{equation}
        \frac1m\sum_{i=1}^m \widetilde T_i(x)
    =
    \left(1-\frac1m\right)x+\frac1m\bar F(x),
    \; \forall x\in\mathbb R^d,
    \end{equation} and therefore $ x=\frac1m\sum_{i=1}^m \widetilde T_i(x)$ is equivalent to
    $x=\bar F(x)$.
}

\begin{proof}
Fix any coordinate $j\in[d]$, and let $i(j)$ be the unique index such that $j\in D_{i(j)}$. By construction,
$[\widetilde T_{i(j)}(x)]_j = [\bar F(x)]_j$. For every $\ell\neq i(j)$, the coordinate $j$ is outside $D_\ell$, so $[\widetilde T_\ell(x)]_j = x_j$. Therefore $\left[\frac{1}{m}\sum_{i=1}^m \widetilde T_i(x)\right]_j
=\frac{1}{m}[\bar F(x)]_j + \frac{m-1}{m}x_j$. Since this holds for every $j\in[d]$, we obtain
\begin{equation}    
\frac{1}{m}\sum_{i=1}^m \widetilde T_i(x)
=
\left(1-\frac{1}{m}\right)x + \frac{1}{m}\bar F(x).
\end{equation}

The fixed-point equivalence follows immediately.
\end{proof}

\subsection{Proof of Proposition~\ref{prop:perform-sd}}\label{subsec:prop3}

\textbf{Proposition 3} (Performance decay of strict-decomposition operator)
\textit{Under the grid environment stated above, consider the residual of strict-decomposition operator
\[
\operatorname{Dev}(\mathcal H^{\rm hm})
:=
\frac{1}{4N}
\sum_{s\in S}\sum_{a\in\mathcal A}
\left[
(\mathcal H^{\rm hm}Q^\star)(s,a)-Q^\star(s,a)
\right]^2 .
\]
Then
\[
\operatorname{Dev}(\mathcal H^{\rm hm})
\ge
\frac{[\sqrt N(\sqrt m-1)-2]_+}{N}
\gamma^{2D_T}R_T^2,
\qquad
[x]_+:=\max\{x,0\}.
\] 
In particular, for fixed \(N,\gamma,R_T,D_T\), the lower bound of error brought by strict-decomposition operator will increase in \(m\).}

\begin{proof}
Define $B_m:=\{(s,a)\in S\times\mathcal A:f(s,a)\notin C_{i(s)}\}$ and $B_m^+
:=
\{(s,a)\in B_m:f(s,a)\neq t\}$. If \((s,a)\notin B_m\), then $f(s,a)\in C_{i(s)}$.
Therefore
\begin{equation}    
(\mathcal H^{\mathrm{hm}}Q)(s,a)
=
r(s,a,f(s,a))+\gamma V_Q(f(s,a))
=
(HQ)(s,a).
\end{equation}

So strict decomposition  introduces no error on interior transitions.

Now take \((s,a)\in B_m\) and write $u=f(s,a)$, then \(u\notin C_{i(s)}\). The strict decomposition update is $(\mathcal H^{\mathrm{hm}}Q)(s,a)
=
r(s,a,u)$, whereas the global Bellman target is
\begin{equation}
(HQ)(s,a)
=
r(s,a,u)+\gamma V_Q(u).
\end{equation}

Evaluating at \(Q^\star\),
\begin{equation}
(\mathcal H^{\mathrm{hm}}Q^\star)(s,a)-Q^\star(s,a)
=
-\gamma V^\star(u).
\end{equation}
Thus the strict decomposition  defect is exactly the missing continuation value.

For pairs in \(B_m^+\), we have \(u\neq t\). By assumption, \(u\) has a
trap-free path to \(t\) of length \(d\le D_T\). Following this path gives return $
\gamma^{d-1}R_T$.
Since \(Q^\star\) is optimal,
\begin{equation}
V^\star(u)
\ge
\gamma^{d-1}R_T
\ge
\gamma^{D_T-1}R_T.
\end{equation}

Hence, for every \((s,a)\in B_m^+\),
\begin{equation}
\left[
(\mathcal H^{\mathrm{hm}}Q^\star)(s,a)-Q^\star(s,a)
\right]^2
\ge
\gamma^{2D_T}R_T^2.
\end{equation}

It remains to count useful boundary pairs. For a \(q\times q\) square partition
of an \(n\times n\) grid, the number of directed artificial boundary actions in
the full grid is
\begin{equation}
B_0(m)
=
4n(q-1)
=
4\sqrt N(\sqrt m-1).
\end{equation}

The absorbing target can remove at most four boundary actions as a source state,
and at most four boundary actions can have successor \(t\), for which
\(V_Q(t)=0\). Therefore $|B_m^+|\ge[B_0(m)-8]_+$.

Using the pointwise lower bound on \(B_m^+\),
\begin{equation}
\operatorname{Dev}(\mathcal H^{\rm hm})
\ge
\frac{|B_m^+|}{4N}\gamma^{2D_T}R_T^2.
\end{equation}

Thus
\begin{equation}
\operatorname{Dev}(\mathcal H^{\rm hm})
\ge
\frac{[B_0(m)-8]_+}{4N}\gamma^{2D_T}R_T^2.
\end{equation}

Substituting $B_0(m)=4\sqrt N(\sqrt m-1)$ 
gives 
\begin{equation}
R_{\mathrm{hm}}(N,m)
\ge
\frac{
[\sqrt N(\sqrt m-1)-2]_+
}{N}
\gamma^{2D_T}R_T^2.
\end{equation}

 \end{proof}

\section{Additional Experiments}

\subsection{SARSA}
\label{supp_sarsa}

\paragraph{Environment and objective.}
We evaluate decentralized energy management on the standard IEEE 9-, 14-, and 30-bus test systems. Each bus is treated as an aggregator equipped with a battery storage system, and the physical transmission-line topology defines the communication and neighborhood graph. The undirected bus-neighbor relations used in the experiments are listed in Table~\ref{tab:ieee-networks}. At time $t$, bus $i$ observes the current time index and its battery level $e_{i,t}\in[0,1]$, and chooses a charging/discharging action $a_{i,t}\in[-1,1]$. The battery evolves according to
\[
e_{i,t+1}
=
\min\{\max(e_{i,t}+a_{i,t},0),1\}.
\]
The objective is to minimize electricity cost while avoiding neighborhood overloads. Given price $p_{i,t}$, fixed load demand $\ell_{i,t}$, battery capacity $\overline E_i$, penalty coefficient $\lambda$, safe threshold $T_i$, and one-hop grid neighborhood $N(i)$ containing $i$, the local cost is
\[
c_{i,t}
=
p_{i,t}(a_{i,t}\overline E_i+\ell_{i,t})
+
\lambda
\left[
\sum_{j\in N(i)}
(a_{j,t}\overline E_j+\ell_{j,t})
-
T_i
\right]_+ .
\]
The global reward is defined by $r_t=-\sum_i c_{i,t}$, following the reward-maximization convention. Since prices and load demands are generated from bounded cyclic profiles with bounded multiplicative noise, the reward is bounded. For a fixed policy $\pi$, the tabular value function satisfies the Bellman fixed point
\[
Q^\pi = H^\pi Q^\pi,
\]
and SARSA provides the sample-path stochastic approximation
\[
Q(s_t,a_t)
\leftarrow
Q(s_t,a_t)
+
\alpha
\big[
r_t+\gamma Q(s_{t+1},a_{t+1})-Q(s_t,a_t)
\big].
\]
\begin{table}[htbp]
\centering
\small
\caption{Topologies (Undirected Graph) of the IEEE 9-, 14-, and 30-bus test systems.}
\label{tab:ieee-networks}
{%
\begin{tabular}{ll|ll|ll|ll}
\toprule
\multicolumn{2}{c}{\textbf{IEEE 9-Bus}} & \multicolumn{2}{c}{\textbf{IEEE 14-Bus}} & \multicolumn{4}{c}{\textbf{IEEE 30-Bus}} \\
\cmidrule(lr){1-2} \cmidrule(lr){3-4} \cmidrule(lr){5-8}
\textbf{Bus} & \textbf{Neighbors} & \textbf{Bus} & \textbf{Neighbors} & \textbf{Bus} & \textbf{Neighbors} & \textbf{Bus} & \textbf{Neighbors} \\
\midrule
1 & 4       & 1  & 2, 5            & 1  & 2, 3                 & 16 & 12, 17 \\
2 & 8       & 2  & 1, 3, 4, 5      & 2  & 1, 4, 5, 6           & 17 & 10, 16 \\
3 & 6       & 3  & 2, 4            & 3  & 1, 4                 & 18 & 15, 19 \\
4 & 1, 5, 9 & 4  & 2, 3, 5, 7, 9   & 4  & 2, 3, 6, 12          & 19 & 18, 20 \\
5 & 4, 6    & 5  & 1, 2, 4, 6      & 5  & 2, 7                 & 20 & 10, 19 \\
6 & 3, 5, 7 & 6  & 5, 11, 12, 13   & 6  & 2, 4, 7, 8, 9, 10, 28& 21 & 10, 22 \\
7 & 6, 8    & 7  & 4, 8, 9         & 7  & 5, 6                 & 22 & 10, 21, 24 \\
8 & 2, 7, 9 & 8  & 7               & 8  & 6, 28                & 23 & 15, 24 \\
9 & 4, 8    & 9  & 4, 7, 10, 14    & 9  & 6, 10, 11            & 24 & 22, 23, 25 \\
  &         & 10 & 9, 11           & 10 & 6, 9, 17, 20, 21, 22 & 25 & 23, 24, 26 \\
  &         & 11 & 6, 10           & 11 & 9                    & 26 & 25 \\
  &         & 12 & 6, 13           & 12 & 4, 13, 14, 15, 16    & 27 & 25, 28, 29, 30 \\
  &         & 13 & 6, 12, 14       & 13 & 12                   & 28 & 6, 8, 27 \\
  &         & 14 & 9, 13           & 14 & 12, 15               & 29 & 27, 30 \\
  &         &    &                 & 15 & 12, 14, 18, 23       & 30 & 27, 29 \\
\bottomrule
\end{tabular}%
}
\end{table}

This experiment directly matches the structural setting studied in the paper. The update associated with bus $i$ is not determined by bus $i$ alone, because the overload term in $c_{i,t}$ depends on the neighboring buses $j\in N(i)$. Thus, if the core of agent $i$ is its own local state-action block, a strict decomposition that only reads bus $i$ removes the cross-boundary variables needed to evaluate the neighborhood penalty. In operator terms, strict decomposition replaces the original Bellman operator $H^\pi$ by a local surrogate whose block update ignores part of the true dependence set.

Core-Halo keeps the same disjoint update ownership: each bus updates only its own local $Q$-table. The difference is that the local observation is augmented with read-only boundary context from the one-hop halo. In our implementation, the core observation of bus $i$ is its current hour and battery level, while the halo observation adds the average battery level of its immediate topological neighbors. Therefore, strict decomposition uses $o^{\mathrm{strict}}_{i,t}=(h_t,e_{i,t})$, whereas Core-Halo uses
\[
o^{\mathrm{halo}}_{i,t}
=
(h_t,e_{i,t},\overline e_{N(i),t}),
\qquad
\overline e_{N(i),t}
=
\frac{1}{|N(i)\setminus\{i\}|}
\sum_{j\in N(i)\setminus\{i\}} e_{j,t}.
\]
This one-hop halo is a compact tabular approximation to the sufficient boundary information induced by the physical grid topology. It preserves local ownership while allowing the SARSA target to condition on the neighboring battery context that affects overload penalties.

The environments use a cyclic 12-time-interval base profile for electricity prices and fixed load demands. The base profiles for all three IEEE systems are reported in Table~\ref{tab:base-profiles}. To model stochastic demand and price volatility, load demands are multiplied by triangular noise supported on $[0.8,1.2]$, and prices are multiplied by triangular noise supported on $[0.9,1.1]$. The same profile-generation procedure is used across centralized, strict-decomposition, and Core-Halo methods.

\begin{table}[h]
\centering
\small
\caption{Cyclic 12-hour base price (\$/MWh) and demand (MWh) profiles used for the IEEE 9, 14, and 30 bus test environments.}
\label{tab:base-profiles}
\begin{tabular}{ccccccc}
\toprule
\multirow{2}{*}{\textbf{Time Index}} & \multicolumn{2}{c}{\textbf{IEEE 9-bus}} & \multicolumn{2}{c}{\textbf{IEEE 14-bus}} & \multicolumn{2}{c}{\textbf{IEEE 30-bus}} \\
\cmidrule(lr){2-3} \cmidrule(lr){4-5} \cmidrule(lr){6-7}
& \textbf{Price} & \textbf{Demand} & \textbf{Price} & \textbf{Demand} & \textbf{Price} & \textbf{Demand} \\
\midrule
1  & 55 & 9.00  & 31 & 9.00  & 29 & 8.85  \\
2  & 48 & 7.05  & 29 & 7.20  & 27 & 7.20  \\
3  & 49 & 7.05  & 29 & 6.90  & 27 & 7.05  \\
4  & 63 & 10.20 & 33 & 10.80 & 31 & 10.65 \\
5  & 63 & 8.85  & 35 & 12.15 & 31 & 10.35 \\
6  & 61 & 6.30  & 35 & 12.15 & 29 & 8.85  \\
7  & 60 & 6.30  & 36 & 12.90 & 30 & 9.15  \\
8  & 65 & 7.95  & 36 & 13.20 & 31 & 10.20 \\
9  & 72 & 12.45 & 38 & 14.55 & 33 & 13.35 \\
10 & 77 & 15.00 & 38 & 15.00 & 34 & 15.00 \\
11 & 69 & 13.95 & 37 & 13.65 & 33 & 14.10 \\
12 & 62 & 11.25 & 34 & 11.10 & 32 & 11.10 \\
\bottomrule
\end{tabular}
\end{table}

All methods are trained with tabular SARSA for 250 episodes over 5 independent random simulation runs. During training, agents use $\epsilon$-greedy exploration with $\epsilon=0.1$ and run for 3600 environment steps per episode. During evaluation, we set $\epsilon=0$ and run deterministic policies for 360 steps per episode. We report two primary metrics over the final 20 episodes of each run: the mean total grid cost and the mean constraint-violation count. The total grid cost is the accumulated electricity payment plus overload penalties, and the violation count records the number of times a neighborhood demand exceeds its safe threshold.
\begin{wraptable}{r}{0.6\linewidth}
\centering
\small
\caption{Comparison of grid cost and constraint violations.}
\label{tab:ieee-eval}
\resizebox{0.6\textwidth}{!}{
\begin{tabular}{cccc}
\toprule
\textbf{Env} & \textbf{Algorithm} 
& \textbf{Cost} ($\downarrow$) 
& \textbf{Violation} ($\downarrow$) \\
\midrule
\multirow{3}{*}{IEEE 9}
& Centralized 
& 23,025.21 $\pm$ 1631.79 
& 8.48 $\pm$ 3.09 \\
& Strict Decom
& 34,169.32 $\pm$ 17196.56 
& 23.74 $\pm$ 20.67 \\
& \cellcolor{gray!15}Core-Halo
& 22,033.87 $\pm$ 1103.14
& 6.74 $\pm$ 2.63 \\
\midrule
\multirow{3}{*}{IEEE 14}
& Centralized 
& 22,914.17 $\pm$ 1535.79 
& 8.59 $\pm$ 2.99 \\
& Strict Decom
& 35,314.99 $\pm$ 14020.83 
& 30.65 $\pm$ 21.48 \\
& \cellcolor{gray!15}Core-Halo
& 22,009.02 $\pm$ 1135.86
& 7.35 $\pm$ 2.80 \\
\midrule
\multirow{3}{*}{IEEE 30}
& Centralized 
& 23,016.88 $\pm$ 1509.36 
& 8.80 $\pm$ 3.27 \\
& Strict Decom
& 42,034.87 $\pm$ 24448.21 
& 36.25 $\pm$ 28.98 \\
& \cellcolor{gray!15}Core-Halo
& 22,237.34 $\pm$ 1334.02
& 7.38 $\pm$ 2.51 \\
\bottomrule
\end{tabular}
}
\end{wraptable}
The results show that strict decomposition incurs substantially higher cost and more constraint violations, indicating that local-only observations are insufficient for anticipating neighborhood overload penalties. This is precisely the boundary-truncation failure predicted by the Core-Halo framework: the strict local SARSA update discards variables that enter the true Bellman target through the reward. Core-Halo reduces this structural gap by adding one-hop read-only context, which keeps the local tabular problem tractable while retaining the most relevant boundary information. Although a fully observable centralized tabular SARSA agent is asymptotically well specified, it suffers in finite training from the curse of dimensionality: the joint state-action space grows rapidly with the number of buses, making exploration sparse under the fixed 250-episode budget. Core-Halo avoids this bottleneck by using small local $Q$-tables while preserving the neighborhood information needed for constraint-aware control.

\subsection{Additional Details for Multi-Intersection Traffic Control}
\label{supp_dqn}
\begin{figure}[h]
    \centering
    \begin{subfigure}[c]{0.42\linewidth}
        \centering
        \includegraphics[height=4.5cm, width=\linewidth, keepaspectratio, trim={0 0 0 35}, clip]{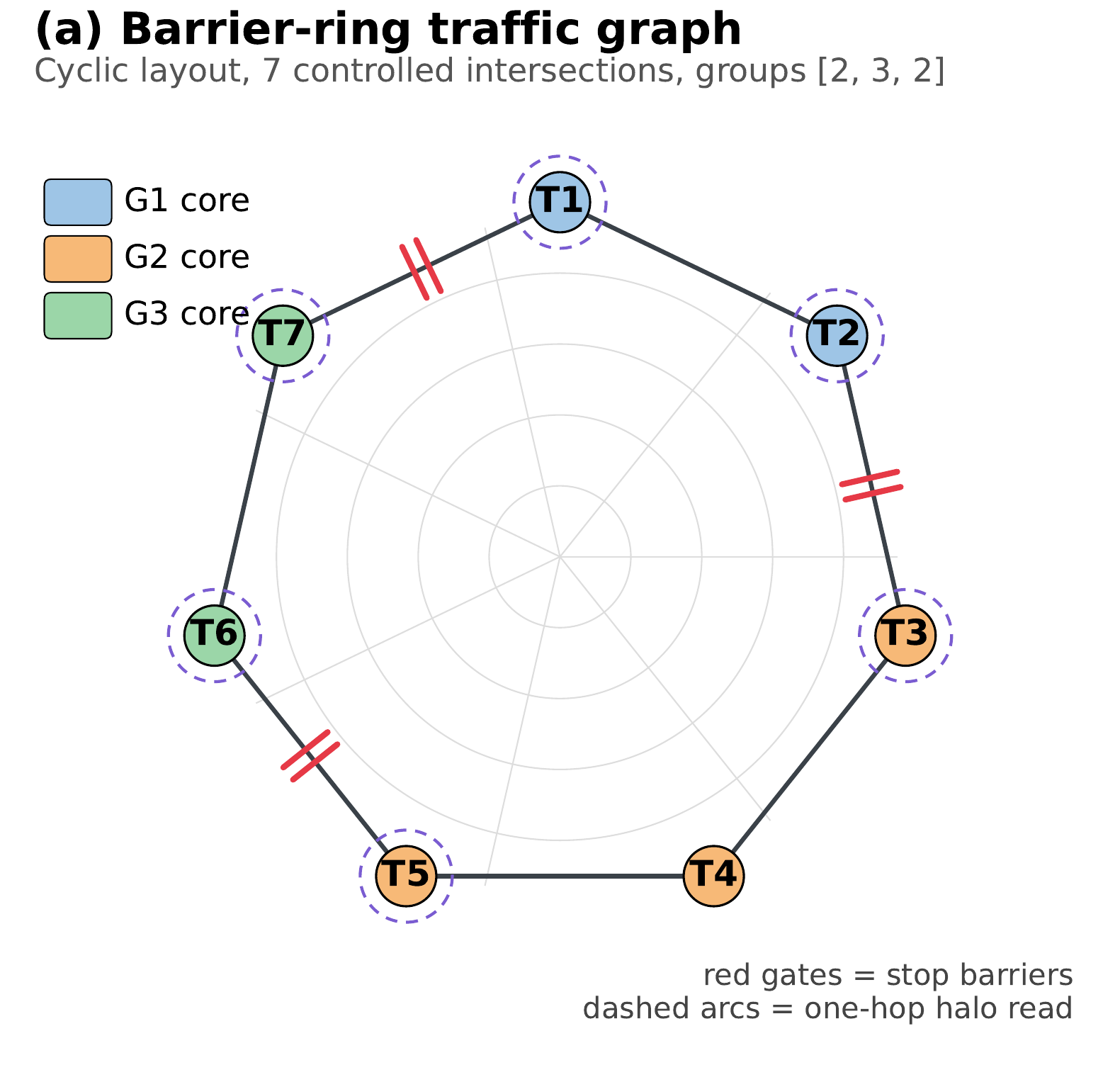}
        \caption{Barrier-ring traffic graph. Cyclic layout with 7 controlled intersections partitioned into groups $[2, 3, 2]$. Red gates mark stop barriers; dashed arcs denote one-hop halo reads.}
        \label{fig:ring-topology}
    \end{subfigure}
    \hfill
    \begin{subfigure}[c]{0.56\linewidth}
        \centering
        \includegraphics[height=6.5cm, width=\linewidth, keepaspectratio, trim={0 0 0 30}, clip]{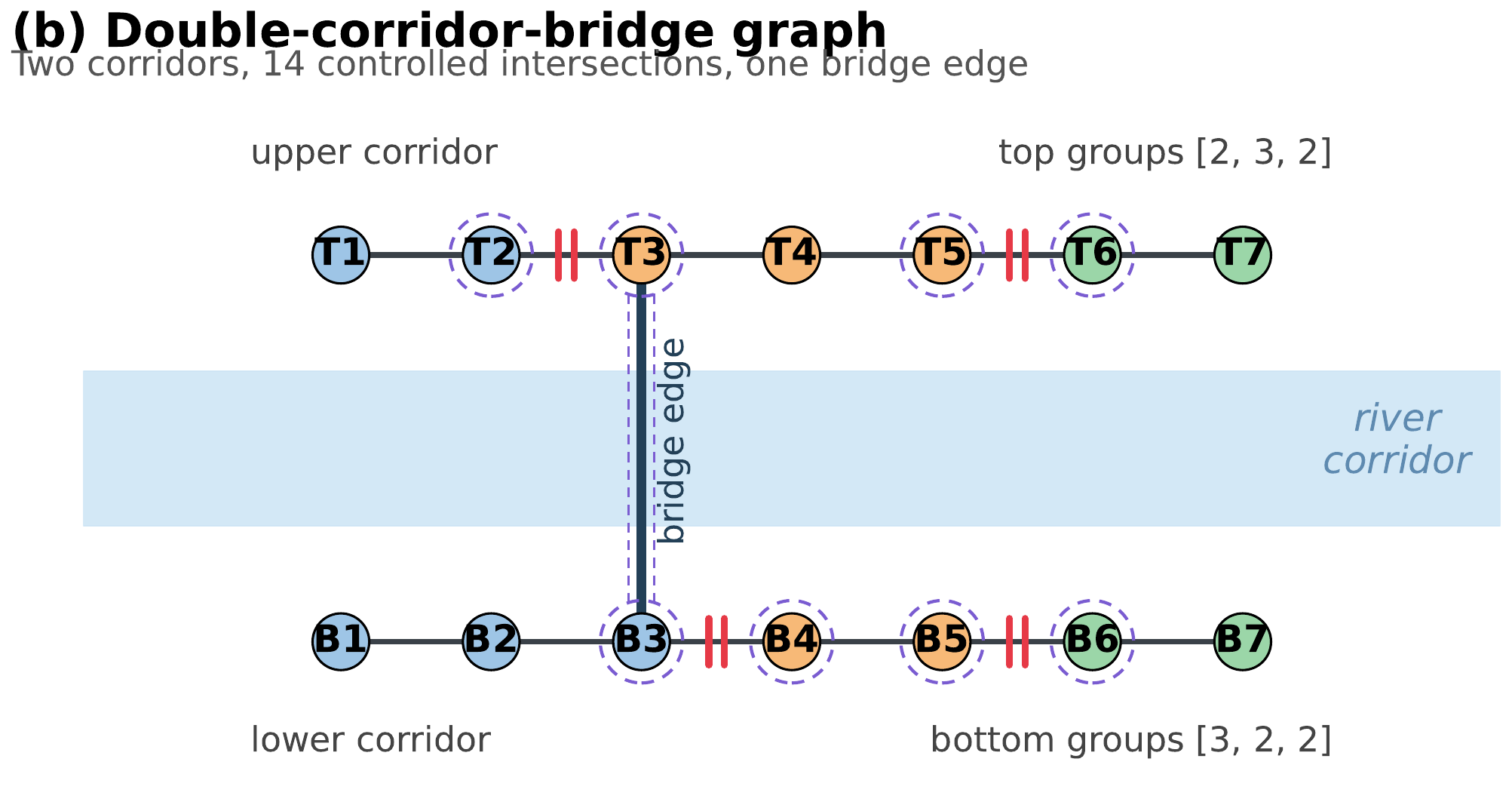}
        \caption{Double-corridor-bridge graph. Two-corridor layout with 14 controlled intersections joined by a single bridge edge; top groups $[2, 3, 2]$, bottom groups $[3, 2, 2]$.}
        \label{fig:bridge-topology}
    \end{subfigure}
    \caption{Two SUMO topologies used to evaluate the strict decomposition , Core-Halo, and centralized DLSA variants. Both layouts share the same group-size structure $[2,3,2]$ within each corridor/ring, but differ in connectivity: (a) is cyclic with a single barrier ring, while (b) introduces a bridge edge linking two parallel corridors.}
    \label{fig:topologies}
\end{figure}
\paragraph{Environment and decomposition.}
The main traffic-control benchmark is a cyclic ring with seven traffic-light intersections partitioned into three decentralized groups of sizes $[2,3,2]$. Consecutive groups are separated by physical stop-sign barriers, with two stop signs between neighboring traffic-light groups and a mandatory five-second halt at each stop. For the decentralized agent graph, we collapse the stop nodes and retain only the traffic-light connectivity, so each traffic light has two graph neighbors in the ring. The Core-Halo variant uses a one-hop halo with radius $r=1$, which gives each traffic-light controller access to its own local observation together with the observations of adjacent traffic lights, while update ownership remains restricted to the assigned core. The strict decomposition  variant uses the same core partition but removes the halo observations, so each local Q-network observes only its own group. We also include a random traffic-network topology with the same decentralized learning protocol to test whether the observed behavior is specific to the regular ring structure or persists under less structured connectivity.

\paragraph{Training and evaluation protocol.}
All traffic-control methods are implemented with DQN and trained for $100$K environment steps using three independent random seeds. Vehicles enter and exit only through traffic-light stubs, with arrival rate $\lambda=0.4$ vehicles per second, corresponding to approximately $1{,}440$ vehicles in each one-hour SUMO simulation. After training, each policy is evaluated with deterministic action selection, i.e., $\epsilon=0$, over five one-hour evaluation episodes for each algorithm and seed. We report three metrics: mean halting count per environment step, mean total waiting time across controlled lanes, and throughput measured as the number of vehicles arriving per simulated hour. Statistical comparisons are computed across seeds using Welch's $t$-test; because only three seeds are used, we interpret non-significant differences as directional evidence rather than definitive statistical separation.

\paragraph{Additional result interpretation.}
In the cyclic-ring environment, Core-Halo obtains the lowest mean waiting time and a lower halting count than strict decomposition, while strict decomposition exhibits substantially larger cross-seed variance. In particular, the standard deviation of waiting time is much larger for strict decomposition than for centralized or Core-Halo, indicating that strict local truncation can make decentralized Q-learning more sensitive to random initialization and sampling. Throughput does not clearly separate the methods in this setting, which is expected because the chosen arrival rate is below the saturation regime where throughput becomes the binding constraint. The open-ended corridor variant provides a stronger boundary-dependence stress test: once the wrap-around edge is removed, endpoint agents have less context, and strict decomposition suffers a clear reward collapse while Core-Halo remains close to centralized performance. This supports the interpretation that the performance gap is driven by missing boundary information rather than by the particular DQN implementation.

\section{Other scenarios}

\label{app:other}

\subsection{Representative settings covered by the unified framework}
\label{app:other-scenarios-table}
The framework is not limited to the experimental domains studied in this paper. Table~\ref{tab:framework-five} summarizes five representative settings that share the same decentralized fixed-point template. The Q-learning, SARSA, and PageRank rows correspond directly to experiments in the main text, while the MPC/QP and implicit-layer rows illustrate broader settings where the same Core-Halo locality principle applies.

\begin{table}[htbp]
\centering
\caption{Representative settings covered by the Core-Halo framework. Panel A identifies the fixed-point equation and the corresponding write/read decomposition. Panel B lists typical well-posedness conditions and representative applications. The structural requirement common to all rows is Definition~\ref{def:Core-Halo}: each block $[\bar F(x)]_{D_i}$ must be computable from the halo variables $x_{S_i}$.}
\label{tab:framework-five}
\scriptsize
\setlength{\tabcolsep}{3.0pt}
\renewcommand{\arraystretch}{1.13}

\begin{tabularx}{\textwidth}{@{}L{2.55cm}L{2.05cm}Y Y@{}}
\toprule
\rowcolor{gray!12}
\multicolumn{4}{@{}l}{\textbf{Panel A: Local fixed-point structure}} \\
\midrule
\textbf{Problem family} &
\textbf{Fixed point} &
\textbf{Owned core $D_i$} &
\textbf{Read-only halo $S_i$} \\
\midrule
\textbf{Value-based RL / Q-learning} &
$Q^\star=H Q^\star$ &
State-action block $C_i\times\mathcal A$. &
One-step successor states of $C_i$ and their action values used in the Bellman target. \\
\addlinespace[2pt]
\textbf{SARSA / Expected SARSA} &
$Q^\pi=H^\pi Q^\pi$ &
State-action block $C_i\times\mathcal A$. &
Successors under the policy, plus neighboring reward or cost variables needed by the local target. \\
\addlinespace[2pt]
\textbf{Sparse MPC / QP splitting} &
$z^\star=T_{\rm solver}(z^\star)$ &
Subsystem-time variables over the prediction horizon. &
Neighboring subsystems and adjacent time slices appearing in dynamics, costs, or coupled constraints. \\
\addlinespace[2pt]
\textbf{PageRank / personalized PageRank} &
$x^\star=\alpha P^\top x^\star+(1-\alpha)v$ &
Graph-node block $D_i$. &
Predecessor neighborhood $D_i\cup\{k:\exists j\in D_i,\;P_{kj}>0\}$. \\
\addlinespace[2pt]
\textbf{Sparse fixed points / implicit layers} &
$z^\star=Az^\star+b$ or $z^\star=T_\theta(z^\star;u)$ &
Coordinates, features, or spatial sites. &
Nonzero pattern of $A$, finite receptive field, or convolution/sparsity neighborhood. \\
\bottomrule
\end{tabularx}

\vspace{0.55em}

\begin{tabularx}{\textwidth}{@{}L{2.55cm}Y Y@{}}
\toprule
\rowcolor{gray!12}
\multicolumn{3}{@{}l}{\textbf{Panel B: Well-posedness and representative uses}} \\
\midrule
\textbf{Problem family} &
\textbf{Typical sufficient condition} &
\textbf{Representative uses and references} \\
\midrule
\textbf{Value-based RL / Q-learning} &
Finite discounted tabular MDP or joint MDP with bounded rewards; $H$ is a $\gamma$-contraction. &
MiniGrid-style navigation, multi-agent control, and traffic-signal learning \cite{watkins1992q,IQL,terry2021pettingzoo}. \\
\addlinespace[2pt]
\textbf{SARSA / Expected SARSA} &
Finite discounted on-policy setting with bounded rewards; Expected SARSA is the stationary mean operator. &
Decentralized markets, smart-grid energy management, and dynamic pricing \cite{vanseijen2009expected,chang2020decentralized}. \\
\addlinespace[2pt]
\textbf{Sparse MPC / QP splitting} &
Sparse strongly convex finite-horizon QP, or a contractive/averaged operator-splitting solver. &
HVAC, microgrids, water networks, platoons, and sparse networked control \cite{odonoghue2013splitting,stellato2020osqp,serale2018mpc}. \\
\addlinespace[2pt]
\textbf{PageRank / personalized PageRank} &
For $\alpha<1$, the damped PageRank map is a contraction; locality follows from graph sparsity. &
Ranking, recommendation, trust propagation, citation graphs, and distributed PageRank \cite{Page1998PageRank,gleich2015pagerank,ishii2010distributed}. \\
\addlinespace[2pt]
\textbf{Sparse fixed points / implicit layers} &
Linear maps with $\|A\|_c<1$, or nonlinear implicit layers with contractive, averaged, or monotone parameterizations. &
Circuit/network solvers, PDE surrogates, inverse problems, and memory-efficient implicit neural layers \cite{bertsekas1983distributed,bai2019deq,winston2020mondeq,ghaoui2021implicit,fung2022jfb}. \\
\bottomrule
\end{tabularx}
\end{table}

The table separates two kinds of assumptions. The Core-Halo bookkeeping only requires the exact locality condition in Definition~\ref{def:Core-Halo}. The contraction, monotonicity, or convexity conditions in Panel B are standard problem-specific assumptions used to ensure that the corresponding fixed point is well posed and that stochastic approximation or deterministic solver iterations are stable. The SARSA row is interpreted through its stationary mean operator $H^\pi$, with tabular SARSA as the corresponding sample-path stochastic approximation. The final row covers both sparse linear maps and nonlinear implicit layers whose local block maps have sparse or finite-receptive-field dependencies.

Two clarifications are worth making. First, the SARSA row is interpreted through its stationary mean operator $H^\pi$; Expected SARSA is the clean mean-field version, while tabular SARSA is the corresponding sample-path stochastic approximation. Second, the last row is intentionally broader than plain linear solvers: vanilla deep equilibrium models are nonlinear implicit fixed-point systems, so to include them rigorously one should phrase the row as ``sparse linear fixed-point solvers + DEQ-type implicit layers'' rather than only ``linear fixed-point solvers.''

On the graph side, distributed randomized PageRank algorithms provide a stochastic/decentralized realization of the PageRank row \cite{ishii2010distributed}. On the implicit-model side, classical DEQ, implicit deep learning, and Jacobian-free training give further support for the last row \cite{ghaoui2021implicit,fung2022jfb}.

\end{document}